\documentclass[conference]{IEEEtran}
\IEEEoverridecommandlockouts
\usepackage{cite}
\usepackage{amsmath,amssymb,amsfonts}
\usepackage{graphicx}
\usepackage{textcomp}
\usepackage{xcolor}
\def\BibTeX{{\rm B\kern-.05em{\sc i\kern-.025em b}\kern-.08em
    T\kern-.1667em\lower.7ex\hbox{E}\kern-.125emX}}

\usepackage{soul} % For \hl and \sethlcolor
\usepackage{xcolor} % For color definitions
\usepackage[colorlinks,linkcolor=black,citecolor=black,urlcolor=black,bookmarks=false,hypertexnames=true]{hyperref}
\usepackage{todonotes}

\newcommand{\thetermlong}{Synergistic Simplex\xspace}%
\newcommand{\theterm}{$\mathcal{SS}$\xspace}%

\newcommand{\ps}{$\mathcal{PS}$\xspace}%

\usepackage{xspace}
\newcommand{\eg}{{\it e.g.,}\xspace}

\newcommand{\etal}{{\it et~al.}\xspace}
\newcommand{\ie}{{\it i.e.}\xspace}

\usepackage{hyperref}

\usepackage{cleveref}
\crefname{assumption}{assumption}{assumptions}
\Crefname{assumption}{Assumption}{Assumptions}
\crefname{theorem}{theorem}{theorems}
\Crefname{theorem}{Theorem}{Theorems}
\crefname{definition}{definition}{definitions}
\Crefname{definition}{Definition}{Definitions}

\usepackage{algorithm}
\usepackage[noend]{algpseudocode} % for \Function, \EndFunction, etc.

\usepackage{amsthm}
\newtheorem{assumption}{Assumption}
\newtheorem{theorem}{Theorem}
\newtheorem{definition}{Definition}

\usepackage{subfig}

\usepackage{tabularx}
\usepackage{array}    % for m{<width>} column types
\usepackage{multirow} % for merged cells

\definecolor{deepgreen}{RGB}{0,128,0}     % dark, rich green
\definecolor{deepyellow}{RGB}{204,153,0}  % muted gold / deep yellow
\definecolor{deepred}{RGB}{153,0,0}       % strong dark red

\usepackage{pifont}
\usepackage{booktabs}

\begin{document}

\title{Synergistic Simplex: Cooperative Runtime Assurance for Safety-Critical Autonomous Systems\\
\thanks{
\textsuperscript{$\star$}Equal contribution.

This material is based upon work supported by
the National Aeronautics and Space Administration (NASA) under the cooperative agreement 80NSSC20M0229 and University Leadership Initiative grant no. 80NSSC22M0070, and
the National Science Foundation (NSF) under grant no. CNS 1932529 and ECCS 2311085.
Any opinions, findings, conclusions or recommendations expressed in this material
are those of the authors and do not necessarily reflect
the views of the sponsors.
}
}

\author{
\IEEEauthorblockN{
Ayoosh Bansal\textsuperscript{$\star$,1}\quad
Mikael Yeghiazaryan\textsuperscript{$\star$,1}\quad
Artyom Khachatryan\textsuperscript{2}\quad
Tianyi Zhu\textsuperscript{3}\\
Hunmin Kim\textsuperscript{4}\quad
Naira Hovakimyan\textsuperscript{1}\quad
Lui Sha\textsuperscript{1}
}

\IEEEauthorblockA{
\textsuperscript{1}University of Illinois Urbana-Champaign (UIUC), Urbana, IL, USA\\
\textsuperscript{2}Center for Scientific Innovation and Education (CSIE), Yerevan, Armenia \\
\textsuperscript{3}California Institute of Technology (Caltech), Pasadena, CA, USA \\
\textsuperscript{4}Mercer University, Macon, GA, USA
}
}

\maketitle

\begin{abstract}

Autonomous systems increasingly rely on machine-learning (ML) components for safety-critical tasks such as perception and control in autonomous vehicles (AVs).
While ML enables essential capabilities, it inevitably exhibits long-tail faults that make it unsuitable for safety-critical tasks.
Runtime assurance (RTA) mitigates this issue by pairing ML components with verifiable safety monitors, \eg Control Simplex and Perception Simplex architectures.
However, the limited performance of safety monitors remains a major bottleneck.

The \thetermlong (\theterm) architecture improves system performance by enabling bidirectional integration between ML components and safety monitors while preserving formal safety guarantees.
The key innovation here is allowing safety monitors to use ML outputs, which is typically prohibited in RTA systems.
We formally derive conditions under which this integration preserves safety and demonstrate the performance benefits.
We present the design, analysis, and evaluation of \theterm for AV obstacle detection.

% 144 words

\end{abstract}


\section{Introduction}
\label{sec:introduction}

Autonomous systems are advancing rapidly, driven by breakthroughs in machine learning (ML) enabling complex autonomy tasks.
Among these systems, autonomous vehicles (AVs) exemplify the reliance on ML for functions essential to the safety goals of the vehicle, \eg perception.
Despite their enormous potential, such ML-enabled systems pose significant safety challenges, as failures in ML components deployed for safety-critical tasks can have severe consequences.

Although ML-based components achieve impressive empirical performance, they
are inherently unverifiable and therefore susceptible to unexpected faults~\cite{pereira2020challenges,goodloe2022assuring,malleswaran2023challenges,zhang2023deep,braiek2025machine}.
Su~\etal~\cite{su2019one} show that even a single-pixel perturbation, imperceptible to humans, can drastically change ML output.
Such fragility raises serious concerns about the safety of autonomous systems built upon fundamentally unreliable ML components~\cite{
jenn2020identifying, 
xu2021machine,
mohseni2022taxonomy,
araujo2024road,
paterson2025safety,
crash_ntsb_tesla_2016_5_7,
crash_ntsb_uber_2018_3_18,
% crash_ntsb_tesla_2018_3_23,
crash_ntsb_tesla_2019_3_1,
NHTSA_SGO_Reports
}.

A solution to this is the runtime assurance (RTA) design, where a simple verifiable component monitors and constrains a complex high-performance component, ensuring the system remains within a safe operating region.
This idea was pioneered by Sha as the Simplex architecture for control systems~\cite{simplex_original}.
Since then, the RTA design has gained prominence, inspiring numerous improvements and adaptations~\cite{simplex1,simplex2,fuller2020run,phan2020neural,phan2017component,mehmood2022black,desai2019soter}.

The growing use of ML-based perception in safety-critical systems motivated the development of Perception Simplex (\ps)~\cite{bansal2022verifiable,perception_simplex}.
\ps adapts the RTA design to perception systems, providing safety guardrails for AVs in the presence of obstacle detection faults in the AV's ML-based perception.
While \ps guarantees safety under specific constraints, it can cause substantial performance degradation due to conservativeness, even in the absence of faults in ML.

\begin{figure}[t]
    \centering
    \includegraphics[width=\columnwidth]{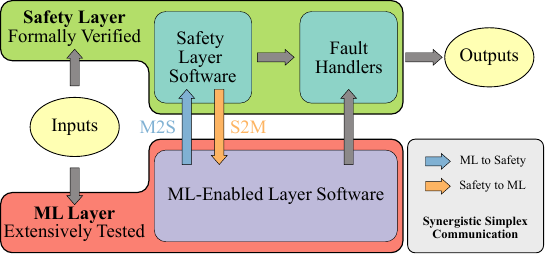}
    \caption{
        Overview of the \textit{\thetermlong (\theterm)} architecture, extended from Perception Simplex~\cite{perception_simplex}.
        As in traditional runtime assurance designs~\cite{simplex_original,simplex1,simplex2,fuller2020run}, the ML Layer executes system's mission, and the Safety Layer enforces deterministic guardrails. 
        \textit{\thetermlong} improves upon prior designs by leveraging bidirectional communication between the layers.
    }
    \label{fig:synergistic_simplex_overview}
\end{figure}

This work addresses the performance limitations of RTA designs such as \ps by introducing a holistic, bidirectional integration between two complementary layers: the extensively validated, yet unverifiable, ML-enabled layer responsible for best-effort task execution, and the verifiable safety layer responsible for enforcing deterministic safety properties. 
In this \textbf{\thetermlong (\theterm)} design, illustrated in Figure~\ref{fig:synergistic_simplex_overview}, these two partially redundant layers cooperate to improve system performance without violating safety guarantees.

Safety certification standards prohibit the use of unverifiable ML generated information by the safety layer~\cite{ISO26262,DO178C}, \eg
the use of lane boundary information generated by an ML component in the safety layer for obstacle detection and collision avoidance in \ps.
\theterm overcomes this restriction under formally stated conditions (\Cref{sec:method:mission-to-safety}), preserving safety guarantees while improving performance (\Cref{sec:results}).

The key contributions of this work are:
\begin{itemize}
    \item \thetermlong (\theterm) architecture that leverages synergistic interactions between the safety and ML layers \emph{to improve overall system performance while preserving formal safety guarantees} (\Cref{sec:method}).

    \item Formal safety analysis of \theterm, deriving the conditions under which bidirectional ML-safety communication preserves the deterministic safety guarantees of \ps (\Cref{sec:method:safety-to-mission,sec:method:mission-to-safety}).

    \item A reference \theterm design for AVs (\Cref{sec:application_av}).
\end{itemize}

% % Paper structure description
% The remainder of this paper is organized as follows.
% \Cref{sec:related_work} reviews related work.
% \Cref{sec:method} presents the core architecture of \theterm.
% \Cref{sec:application_av} illustrates the application of \theterm to AV systems through a concrete design example.
% \Cref{sec:experiments,sec:results} describe the experimental setup and report results demonstrating the effectiveness of \theterm.
% \Cref{sec:limitations_and_future_work} discusses limitations and outlines directions for future work.
% Finally, \Cref{sec:conclusion} concludes the paper.

\section{Related Work}
\label{sec:related_work}

This section reviews relevant prior work and positions our contributions within safe autonomy research.

\subsection{Control Simplex}
\label{sec:related_work:control_simplex}

% Redundancy has long been used to achieve fault tolerance in safety-critical cyber-physical systems (CPS) through replicated components or diverse implementations~\cite{ishii1986fault,lala1994architectural,blanke2006diagnosis}. 
% However, traditional computational redundancy remains vulnerable to correlated and design-level faults, which are common in learning-enabled components~\cite{6629559,wiersma2017safety,jha2019ml}. 
% This limitation motivated the use of functional redundancy, in which heterogeneous algorithms or sensing modalities are combined to reduce shared failure modes.
% The Simplex architecture builds directly on this idea.
Redundancy has long been used to achieve fault tolerance in safety-critical cyber-physical systems (CPS) through replication or diverse implementations~\cite{ishii1986fault,lala1994architectural,blanke2006diagnosis}. 
However, traditional computational redundancy remains vulnerable to correlated and design-level faults common in learning-enabled components~\cite{6629559,wiersma2017safety,jha2019ml}. 
This limitation motivates functional redundancy, where heterogeneous algorithms or sensing modalities reduce shared failure modes.
The Simplex architecture builds on this idea.

The Simplex architecture~\cite{simplex_original,simplex1,simplex2} is foundational to our work. 
It pairs a high-performance but unverified controller with a simple, verifiable backup monitored at runtime. When the system approaches an unsafe region, control is switched to the safety controller.

Control Simplex has been applied broadly in CPS and autonomy, including variants such as R-Simplex~\cite{wang2018rsimplex}, SL1-Simplex~\cite{mao2023sl1}, and other extensions~\cite{musau2022using}.  
Neural Simplex~\cite{phan2020neural} introduces limited communication from the safe to the complex controller to reduce unnecessary fallbacks, but this interaction is restricted to low-level control.
In contrast, \theterm operates at the perception layer and enables richer ML-safety cooperation.  
We discuss its perception layer predecessor, Perception Simplex~\cite{perception_simplex}, in \Cref{sec:related_work:perception_simplex}.

\subsection{Perception Simplex}
\label{sec:related_work:perception_simplex}
\label{sec:related_work:verifiable_obstacle_detection}
\label{eq:vod_obstacle_detection_requirement}

\begin{figure}
    \centering
    \includegraphics[width=0.8\columnwidth]{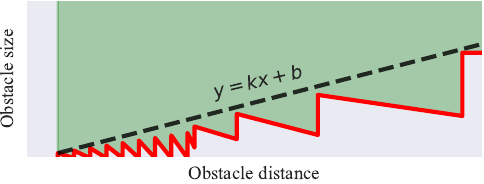}
    \caption{
    Detectability region for LiDAR-based obstacle detection, adapted from~\cite{bansal2022verifiable}. 
    Obstacles above the dashed line are guaranteed to be detected based on their height and distance.
    }
    \label{fig:detectability_region}
\end{figure}

Perception Simplex (\ps)~\cite{perception_simplex} provides deterministic safety guarantees for AVs by separating the autonomy stack into an unverifiable ML layer and a verifiable safety layer.
The safety layer implements a classical LiDAR-based obstacle-existence detector built on Verifiable Obstacle Detection (VOD)~\cite{bansal2022verifiable}, which derives geometric conditions under which the Depth Clustering algorithm~\cite{bogoslavskyi16iros,bogoslavskyi17pfg} is guaranteed to detect an obstacle.
As illustrated in \Cref{fig:detectability_region}, detection is guaranteed whenever an obstacle at distance~$x$ has height~$y$ satisfying $y \ge k x + b$, where $k$ and $b$ depend on sensor and algorithm parameters.
This bound enables a provably safe maximum speed: the ML layer provides high-performance perception and planning, while the safety layer ensures collision avoidance whenever the vehicle operates within this limit.

At runtime, the safety layer monitors the ML layer and overrides when it detects a safety-critical obstacle missed by the latter, ensuring collision avoidance under faulty conditions in obstacle-existence detection. However, \ps covers only obstacle existence (not classification) and its override action is restricted to full stop, which can induce conservative behavior. \theterm extends \ps by enabling carefully constrained bidirectional communication between the two layers, preserving its formal guarantees while reducing conservativeness.

\subsection{Safety Approaches for ML Perception}
\label{sec:related_work:priority_inversion_in_dnns}

% Several complementary directions have been explored to improve the safety of ML-based perception in autonomous systems.
Beyond runtime assurance-based designs, several complementary directions have been proposed to improve the safety of ML-based perception in autonomous systems.

Priority inversion in ML-based perception arises when networks allocate uniform computation across the input, delaying processing of safety-critical regions~\cite{priority_inversion}. Prior work mitigates this through externally proposed or heuristic region partitioning: Liu~\etal~\cite{liu2020removing,liu2021real}, Hu~\etal~\cite{hu2021exploring,hu2022real}, and Chen~\etal~\cite{chen2021lidar} use external agents to rank camera or LiDAR regions. Tung~\etal~\cite{tung2022irrelevant} remove irrelevant pixels before focused convolution, and Kang~\etal~\cite{kang2022dnn} and Liu~\etal~\cite{liu2022self} use regions of interest to focus the DNNs.

Uncertainty estimation offers another avenue, enabling ML models to quantify their own confidence in perception outputs. Kendall and Gal~\cite{NIPS2017_2650d608} distinguish aleatoric and epistemic uncertainty in deep learning for perception tasks, while Lakshminarayanan~\etal~\cite{NIPS2017_9ef2ed4b} propose deep ensembles as a scalable approach to predictive uncertainty estimation. Araujo~\etal~\cite{e26080634} survey uncertainty methods specifically in the context of autonomous driving perception. While these approaches can flag potentially unreliable outputs, they do not provide deterministic safety guarantees.

Formal verification of neural networks attempts to provide stronger correctness guarantees by proving properties of network behavior. Katz~\etal~\cite{katz2017reluplex} introduce Reluplex, an SMT-based solver for verifying deep neural networks, and Huang~\etal~\cite{huang2017safety} develop safety verification methods for DNNs. However, these techniques remain computationally expensive and do not scale to the large networks used in modern autonomy stacks.

In contrast, \theterm does not attempt to verify or constrain the ML layer directly. Instead, it preserves formal safety guarantees by pairing the ML layer with a verifiable safety layer, tolerating ML faults at runtime rather than eliminating them at design time.
\section{Method}
\label{sec:method}

We now present the design of the proposed \theterm framework. 
We emphasize that \theterm is not a standalone autonomy stack. 
It is an architectural augmentation that operates atop an existing, fully functional ML layer, providing verifiable safety coordination without replacing any autonomy components.
Importantly, \theterm is agnostic to the underlying ML model and makes no assumptions about the architecture or algorithm of the ML layer, operating atop any module that proposes actions.

At its core, \theterm extends the \ps architecture~\cite{perception_simplex} by generalizing its unidirectional supervisory structure into a \emph{bidirectional} information exchange between the mission and safety layers, as illustrated in \Cref{fig:synergistic_simplex_overview}. 
Classical Simplex designs rely on a one-way supervisory link, where the safety layer monitors ML layer actions through \emph{Fault Handlers} and intervenes only when safety conditions defined by the \emph{Safety Layer Software} are violated. 
In contrast, \theterm establishes a two-way information exchange architecture between the safety and ML layers that allows the software of both layers to cooperate, improving overall performance beyond what \ps achieves while preserving formal safety guarantees.
The following subsections formalize the Safety-to-ML (S2M) and ML-to-Safety (M2S) communication links and establish that, under their respective assumptions, both extensions preserve the safety guarantees of the original \ps architecture.

\subsection{Safety Layer to ML Layer Communication}
\label{sec:method:safety-to-mission}

The safety-to-ML (S2M) link, shown in \Cref{fig:synergistic_simplex_overview}, enables the safety layer to convey verified information back to the ML layer. 
When a violation is detected by the \emph{Fault Handlers}, the safety layer communicates the violated constraints, allowing the ML layer to recompute actions within verified bounds. 
Formally, when the ML layer proposes an action $a_{M}$ that violates safety constraints $C_S$ specified by the \emph{Safety Layer Software}, the safety layer does not immediately override it. Instead, the fault handler communicates $C_S$ to the ML layer in an acceptable form $f_\text{S2M}(C_S)$, where $f_\text{S2M}$ transforms $C_S$ to match the ML layer input. It then requests recomputation of the action under these constraints. Let $\hat{a}_{M}$ denote the recomputed action. The system proceeds as follows:
\begin{itemize}
    \item If the recomputed action $\hat{a}_{M}$ satisfies all safety constraints $C_S$, the system adopts $\hat{a}_{M}$ as the final action.
    \item If no such safe recomputation is possible, the safety layer overrides with its own verifiable safe action $a_\text{safe}$.
\end{itemize}
\ps can be viewed as a special case of this algorithm (see \Cref{alg:s2m}) in which the recomputed action is fixed to $a_\text{safe}$.

\begin{algorithm}
\caption{Safety-to-ML Cooperative Action Selection}
\label{alg:s2m}
\begin{algorithmic}[1]
\Require Mission proposal $a_{M}$, safety constraints $C_S$
\Ensure Final system action $a^{*}$

\If{$a_{M}$ satisfies $C_S$}
    \State $a^{*} \gets a_{M}$
\Else
    \State Convert $C_S$ into ML layer guardrails $f_\text{S2M}(C_S)$
    \State Request recomputation $\hat{a}_{M}$ under $f_\text{S2M}(C_S)$
    \If{$\hat{a}_{M}$ satisfies $C_S$}
        \State $a^{*} \gets \hat{a}_{M}$
    \Else
        \State $a^{*} \gets a_\text{safe}$ \Comment{Fallback safe override}
    \EndIf
\EndIf

\State \Return $a^{*}$
\end{algorithmic}
\end{algorithm}

We formalize the S2M selection. Let $\mathcal{X}$ be the state space and $\mathcal{A}$ the action space. For a given state $x\in\mathcal{X}$, let
\begin{align}
    \mathcal{S}(x)\;=\;\{\,a\in\mathcal{A}\mid \varphi(x,a)\,\}
\end{align}
denote the set of actions that satisfy the safety constraints $C_S$ (\ie the predicate $\varphi(x,a)$ encodes $C_S$). The safety layer provides a verification procedure
\begin{align}
    \mathrm{Ver}(x,a)\in\{\mathrm{true},\mathrm{false}\}
\end{align}
used to test membership in $\mathcal{S}(x)$.

\begin{assumption}[Sound Verification]\label{asm:sound}
For all $x\in\mathcal{X}$ and $a\in\mathcal{A}$, if $\mathrm{Ver}(x,a)=\mathrm{true}$ then $a\in\mathcal{S}(x)$.
\end{assumption}

\begin{assumption}[Safe Maneuverable State Space]
\label{asm:override}
There exists a subset $\mathcal{X}^{S} \subseteq \mathcal{X}$ such that:

\begin{enumerate}
    \item For every $x \in \mathcal{X}^{S}$, the safety layer can produce at least one verifiable safe action 
    $a_{\mathrm{safe}} \in \mathcal{S}(x)$ that ensures the next state satisfies
    \[
        x(t+1) \in \mathcal{X}^{S}.
    \]
    \item The system is initialized in a state $x(0) \in \mathcal{X}^{S}$.
\end{enumerate}
Thus, $\mathcal{X}^{S}$ is an invariant set under the safe action $a_{\mathrm{safe}}$.
\end{assumption}

\begin{assumption}[Validity Domain]\label{asm:validity}
The \emph{Safety Layer Software} operates within its stated validity constraints, so that Assumptions~\ref{asm:sound}–\ref{asm:override} hold at the current $x$.
\end{assumption}

Given a mission proposal $a_M$, Algorithm~\ref{alg:s2m} returns
\begin{align}
    a^{*}\;\in\;\{\,a_M,\ \hat a_M,\ a_\text{safe}\,\},
\end{align}
where $\hat a_M$ is a recomputation under guardrails derived from $C_S$.

\begin{theorem}[Safety of S2M]\label{thm:s2m-safety}
Under \Cref{asm:sound,asm:override,asm:validity}, \Cref{alg:s2m} always returns an action $a^{*}\in\mathcal{S}(x)$.
\end{theorem}

\begin{proof}
We argue by contradiction. Fix $x\in\mathcal{X}$ and suppose Algorithm~\ref{alg:s2m} returns $a^{*}$ with $a^{*}\notin\mathcal{S}(x)$.
By the algorithm’s structure there are three mutually exclusive cases:

\smallskip
\noindent\emph{Case 1: $a^{*}=a_M$.}
The algorithm returns $a_M$ only if $\mathrm{Ver}(x,a_M)=\mathrm{true}$. By Assumption~\ref{asm:sound}, this implies $a_M\in\mathcal{S}(x)$, contradicting $a^{*}\notin\mathcal{S}(x)$.

\smallskip
\noindent\emph{Case 2: $a^{*}=\hat a_M$.}
The algorithm returns $\hat a_M$ only if $\mathrm{Ver}(x,\hat a_M)=\mathrm{true}$. By Assumption~\ref{asm:sound}, $\hat a_M\in\mathcal{S}(x)$, again a contradiction.

\smallskip
\noindent\emph{Case 3: $a^{*}=a_\text{safe}$.}
The algorithm falls back to $a_\text{safe}$ only if both $\mathrm{Ver}(x,a_M)=\mathrm{false}$ and $\mathrm{Ver}(x,\hat a_M)=\mathrm{false}$. By Assumption~\ref{asm:override}, $a_\text{safe}\in\mathcal{S}(x)$, contradicting $a^{*}\notin\mathcal{S}(x)$.

\smallskip
All cases contradict the supposition. Hence $a^{*}\in\mathcal{S}(x)$.
\end{proof}

\subsection{ML Layer to Safety Layer Communication}
\label{sec:method:mission-to-safety}

The ML-to-safety (M2S) link allows the safety layer to incorporate selected outputs from the ML layer while preserving the formal safety guarantees of the architecture. 
Concretely, in this paper we focus on using the ML layer's \emph{lane-detection} output within the safety layer. 
Lane detection is typically reliable and informative enough to enable more effective decision-making in the safety layer.

Traditionally, safety-critical designs avoid incorporating unverifiable ML outputs into the safety layer, precisely to prevent such information from influencing verified safety guarantees.
Arbitrary use of ML information may introduce new fault-propagation paths, potentially undermining existing safety guarantees. 
This subsection formalizes the conditions under which such use is admissible and proves that \theterm retains the guarantees of \ps when using ML layer outputs that satisfy these conditions.
We begin by formalizing the notion of a fault.

\begin{definition}[Fault]
\label{def:fault}
A \emph{fault} is the adjudged or hypothesized cause of an error~\cite{1335465}. 
\end{definition}

In our setting, faults correspond to deviations such as misclassifications or missed detections that may propagate to downstream functions and influence system-level safety.

\begin{figure}
    \centering
    \includegraphics[width=\columnwidth]{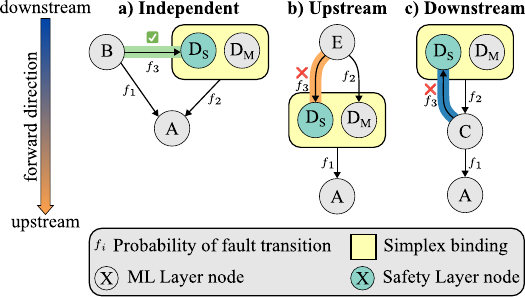}
    \caption{Simplified dependency graph of the autonomy stack, showing independent, upstream, and downstream ML-to-safety-layer links.}
    \label{fig:graph_model}
    \label{fig:graph_model:ml}
    \label{fig:graph_model:ps}
    \label{fig:graph_model:ss}
\end{figure}

\subsubsection{Graph Model}

We begin by formalizing the probabilistic fault model that underlies the M2S analysis. 
Let the autonomy stack be represented as a directed acyclic graph (DAG) $G = (V, E)$, where ``directed'' indicates that each edge has a direction of information flow.
Each node $v_i \in V$ denotes a computational function (\eg perception, planning, or safety), and each edge $(v_i, v_j) \in E$ represents both an information flow and a potential fault-propagation path.  
In this model, edges encode \emph{depend} relations in the sense of system safety: a function that depends on another may inherit its faults, whereas mere use relations without fault-propagation semantics are intentionally excluded.
Each node may produce faulty outputs, modeled by binary random variables $X_i \in \{0,1\}$, where $X_i = 1$ indicates that node (function) $v_i$ exhibits a fault as defined in \Cref{def:fault}.  
This abstraction allows us to reason about how faults propagate through the autonomy stack.

Our formulation is inspired by the deterministic fault tree introduced in \ps~\cite{perception_simplex}, but generalizes it by allowing probabilistic propagation along each edge of $G$. 
For illustration, and to provide intuition for the nodes depicted in \Cref{fig:graph_model}, one may imagine $A$ as a downstream planner, $B$ as an independent lane detector, $C$ as an obstacle-refinement module, $(D_M, D_S)$ as the mission and safety layer obstacle detectors, and $E$ as an upstream preprocessing stage. These mappings are merely illustrative and not inherent to the formal model.

\begin{assumption}[Fault Propagation]
\label{asm:faults_graph}
A fault propagates from a faulty node $X$ to a dependent node $Y$ along each edge with probability $f_{XY} \in [0,1]$. 
\end{assumption}

\begin{assumption}[Downstream Bottleneck]
\label{asm:control_bottleneck}
The dependency graph $G$ contains a unique downstream node $u\in V$
such that every maximal directed path in $G$ terminates at $u$.
\end{assumption}

\Cref{asm:faults_graph} generalizes the fault tree model from \ps~\cite{perception_simplex}, which is recovered when all $f_{XY} = 1$.
\Cref{asm:control_bottleneck} reflects the structure of modern autonomy stacks, in which computations generally converge to a single downstream planning or control command executed by the vehicle.

We additionally impose a structural assumption on the autonomy stack. Modern perception-planning pipelines are typically feed-forward, with information flowing from sensing to control without forming feedback cycles between nodes. The following assumption formalizes this property.
\begin{assumption}[Directed Acyclic Graph Structure]
\label{asm:acyclic_graph}
The dependency graph $G$ is a directed acyclic graph (DAG). That is, it contains no directed cycles.
\end{assumption}

Given the directed acyclic graph $G = (V, E)$, we categorize all ML layer functions according to their structural relationship with the ML layer node $D_M$ that is covered by the safety layer node $D_S$ (see \Cref{fig:graph_model}).
\begin{definition}[Independent, Upstream, and Downstream Functions]
\label{def:task_classes}
Let $D_M \in V$ denote the ML layer module (\eg \emph{obstacle detection}) monitored by the safety layer. 
For any node (function) $X \in V$ in the graph $G$:
\begin{itemize}
    \item \emph{\underline{I}ndependent \underline{F}unctions (IF):}  
    $X$ is independent of $D_M$ if no directed path connects $X$ and $D_M$ in either direction.

    \item \emph{\underline{U}pstream \underline{F}unctions (UF):}  
    $X$ is upstream of $D_M$ if there exists a directed path $X \leadsto D_M$ in $G$.
    
    \item \emph{\underline{D}ownstream \underline{F}unctions (DF):}  
    $X$ is downstream of $D_M$ if there exists a directed path $D_M \leadsto X$ in $G$.
\end{itemize}
\end{definition}

We base our analysis on a simplified dependency graph for clarity. All subsequent results extend to any autonomy graph that satisfies \Cref{asm:faults_graph,asm:control_bottleneck,asm:acyclic_graph}, since the analysis depends only on structural relations (upstream, downstream, independent) rather than on the specific number of modules.

We now analyze, in turn, whether each class of functions (independent, upstream, and downstream) can be safely incorporated into the safety layer through ML-to-safety use, thus augmenting \ps into the proposed \theterm architecture.
The three use configurations shown in \Cref{fig:graph_model} correspond to the independent, upstream, and downstream cases analyzed in \Cref{sec:method:mission-to-safety:independent,sec:method:mission-to-safety:upstream,sec:method:mission-to-safety:downstream} respectively.

\subsubsection{Independent Functions}
\label{sec:method:mission-to-safety:independent}

We begin by analyzing how using independent nodes in the safety layer influences system-level correctness in the architectures considered in this paper: ML baseline, \ps, and \theterm.  
Intuitively, independent functions correspond to disjoint functional branches of the autonomy stack. For example, lane and obstacle detection operate on the same sensor inputs but generate outputs that do not depend on one another (see \Cref{fig:graph_model}a).  
This structural separation ensures that information used from one module cannot propagate faults into another module's correctness domain.

In the unprotected machine-learning baseline, the safety layer node $D_S$ is absent, and all ML layer functions may propagate faults to the downstream node~$A$.  
For illustration, consider the two disjoint upstream branches $B \leadsto A$ and $D_M \leadsto A$.  
Let $f_1 = f_{BA}$ denote the propagation probability along the $B \leadsto A$ edge, and let $f_2 = f_{D_M A}$ denote the propagation probability along the $D_M \leadsto A$ path.  
The resulting failure probability for $A$ is
\begin{align}
% p_\text{ML}(A=1 \mid B=0, D_M=0) &= 0, \\
p_\text{ML}(A=1 \mid B=1, D_M=0) &= f_1, \\
p_\text{ML}(A=1 \mid B=0, D_M=1) &= f_2, \label{eq:ml_od_failure} \\
p_\text{ML}(A=1 \mid B=1, D_M=1) &= f_1 + f_2 - f_1 f_2. \label{eq:full_failure_risk_ml}
\end{align}

\ps introduces a verifiable safety node $D_S$ associated with the ML layer node $D_M$.  
Faults from $D_M$ can affect $A$ only if both $D_M$ and $D_S$ are faulty.  
Under normal operation, $D_S$ monitors the outputs of $D_M$ and overrides them when a verified violation is detected.

\begin{assumption}[Safety Layer Reliability]
\label{asm:safety_layer_node}
The safety layer (node~$D_S$) operates within its verified domain and therefore never produces faults, \ie $p(D_S=1)=0$.
\end{assumption}

% Under \Cref{asm:safety_layer_node}, the fault-propagation probabilities in \ps simplify to
Under \Cref{asm:safety_layer_node}, the probabilities in \ps simplify to
\begin{align}
% p_\text{PS}(A=1 \mid B=0, D_M=0) &= 0, \\
p_\text{PS}(A=1 \mid B=1, D_M=0) &= f_1, \\
p_\text{PS}(A=1 \mid B=0, D_M=1) &= 0, \label{eq:ps_guarantees} \\
p_\text{PS}(A=1 \mid B=1, D_M=1) &= f_1. \label{eq:full_failure_risk_ps}
\end{align}

Hence, \ps guarantees that no unsafe behavior arises when \emph{only} the corresponding ML layer module fails --- in this case, the obstacle detector $D_M$.  
Trivially, it also guarantees safety when none of the ML layer functions fail.

In \theterm, the safety layer (node $D_S$) may use information from an independent ML layer function (\eg node $B$), as illustrated in \Cref{fig:graph_model:ss} via the added $B \to D_S$ edge.  
This configuration corresponds to disjoint branches in the fault tree, where the used information originates from a module whose faults are causally decoupled from those covered by the Simplex relation.

To clarify why the M2S link preserves the formal guarantees of the original \ps, we highlight the standard independence assumptions that underlie Simplex architectures. 
We then show how they naturally extend to lane detection in \theterm.
It is standard in Simplex architectures to assume correctness of certain independent components.
For example, \ps assumes that the braking system operates correctly during an emergency stop: if the brakes were to malfunction, the vehicle could fail to decelerate even when the safety layer issues a verified override. Such failures fall outside the Simplex guarantee envelope because they concern independent subsystems that the safety layer neither monitors nor controls.

The same argument applies to independent software functions used by \theterm. 
Lane detection plays the role of such an essential and independent component: its correctness is required for safe autonomous driving regardless of whether its output is used in the safety layer. If lane detection were faulty, hazards such as drifting out of the lane or entering opposing traffic could occur irrespective of the M2S link. 
At the same time, lane information is precisely what determines whether a detected obstacle lies in the ego lane, where it is safety-critical, or in a neighboring lane, where it poses less immediate threat. 
Thus, assuming correct lane detection is analogous to assuming correct braking behavior in classical Simplex designs. 
Moreover, incorporating lane information allows \theterm to safely reduce conservativeness by discounting obstacles that are not in the ego lane.

% We now make this intuition precise and prove that the independence condition is enough for \theterm to inherit the safety guarantees of \ps.
% When transitioning from \ps to \theterm, we drop \Cref{asm:safety_layer_node}, since the safety layer ($D_S$) now ingests ML layer information.  
% Consequently, $D_S$ can no longer be assumed fault-free, as its correctness depends partially on unverifiable inputs.
% Nevertheless, we show below that, under this independence condition, \theterm retains the formal safety guarantees of \ps.

We now formalize this intuition and show that the independence condition suffices for \theterm to inherit the safety guarantees of \ps.
When transitioning from \ps to \theterm, we drop \Cref{asm:safety_layer_node}, since the safety layer ($D_S$) now ingests ML-layer information and can no longer be assumed fault-free. 
Nevertheless, under this independence condition, \theterm retains the formal safety guarantees of \ps.

\begin{theorem}[Safety of \theterm under Independent M2S Use]
\label{thm:m2s-safety}
When \theterm uses ML layer outputs from functions independent of the one covered by the Simplex relation, the system preserves the safety guarantees of \ps. Formally,
\begin{align}
    p_\text{IF}(A=1 \mid B=0, D_M=1) = 0.
\end{align}
\end{theorem}

\begin{proof}
The M2S link introduces a new edge $(B, D_S)$, allowing ML information to influence the safety layer.  
Let $f_3 = f_{BD_S} \in [0,1]$ denote the probability that a fault propagates along this new edge.  
The modified fault propagation model yields:
\begin{align}
p_\text{IF}(A=1 \mid B=1, D_M=0) &= f_1, \\
p_\text{IF}(A=1 \mid B=0, D_M=1) &= 0, \\
p_\text{IF}(A=1 \mid B=1, D_M=1) &= f_1 + f_2 f_3 - f_1 f_2 f_3. \label{eq:it_full_failure_risk_ss}
\end{align}
Since $p_\text{IF}(A=1 \mid B=0, D_M=1) = 0$, the safety guarantee against failures in $D_M$ remains identical to \ps.  
Thus, using outputs of independent functions does not affect safety against faults in the Simplex-covered function.
\end{proof}

\paragraph{Risk Bounds Under Multi-Function ML Layer Failures}

It is noteworthy that when the entire input stack of the ML layer fails (\ie both $B$ and $D_M$ are faulty), \ps reduces but does not eliminate the risk of system-level failure.  
It reduces the failure probability from $f_1 + f_2 - f_1 f_2$ to $f_1$ according to \Cref{eq:full_failure_risk_ml,eq:full_failure_risk_ps}. \theterm also reduces the failure probability, but to a lesser extent.
This failure mode lies outside the formal guarantee envelope of \ps and \theterm and therefore does not violate their stated safety properties.

\begin{theorem}[Comparative Risk Bounds]
\label{thm:m2s-risk}
Let $F := \{B=1, D_M=1\}$ denote the event that both ML layer inputs are faulty.
Under $F$, the following inequality holds:
\begin{align}
    p_{\text{PS}}(A=1 \mid \text{F}) \leq p_{\text{IF}}(A=1 \mid \text{F}) \leq p_{\text{ML}}(A=1 \mid \text{F}),
\end{align}
where
\begin{align}
p_{\text{PS}}(A=1 \mid B=1, D_M=1) &= f_1, \\
p_{\text{IF}}(A=1 \mid B=1, D_M=1) &= f_1 + f_2 f_3 - f_1 f_2 f_3, \\
p_{\text{ML}}(A=1 \mid B=1, D_M=1) &= f_1 + f_2 - f_1 f_2.
\end{align}
\end{theorem}

\begin{proof}
Since all $f_i \in [0,1]$,
\begin{align}
    0 \leq f_2 f_3 (1 - f_1) \leq f_2 (1 - f_1).
    \label{eq:comparative_risk_bounds}
\end{align}
Adding $f_1$ to each term yields
\begin{align}
    f_1 \leq f_1 + f_2 f_3 (1 - f_1)
        \leq f_1 + f_2 (1 - f_1),
\end{align}
that is,
\begin{align}
    f_1
    \leq
    f_1 + f_2 f_3 - f_1 f_2 f_3
    \leq
    f_1 + f_2 - f_1 f_2.
\end{align}
Identifying these three quantities as
$p_{\text{PS}}(A=1 \mid F)$,
$p_{\text{IF}}(A=1 \mid F)$,
and
$p_{\text{ML}}(A=1 \mid F)$
gives the claimed inequality.
\end{proof}

These results establish that \theterm retains the deterministic safety guarantees of \ps for independent ML layer function use, while being verifiably safe in the scenario where only the corresponding ML layer module fails. 
At the same time, \theterm tolerates a marginally higher overall system risk when the entire mission input stack (\eg $B$ and $D_M$) fails.  
As already mentioned, this failure mode is beyond the formal guarantee envelope of both \ps and \theterm, and the comparison here reflects an extended robustness analysis rather than a weakening of the stated safety properties.
This intentional relaxation trades a small amount of global robustness for improved performance and reduced conservativeness, as empirically confirmed in \Cref{sec:results}, yielding a more balanced safety-efficiency trade-off while preserving the verified safety envelope established by \ps as stated in \Cref{thm:m2s-safety}.  
% A concise qualitative comparison between architectures is summarized in \Cref{tab:comparison}.

% \begin{table}[t]
%     \centering
%     \caption{
%         Comparison of safety guarantees under obstacle-detection faults and other mission-layer faults.
%     }
%     \label{tab:comparison}
%     \renewcommand{\arraystretch}{1.5}

%     \begin{tabularx}{\columnwidth}{
%         m{3.4cm}
%         >{\centering\arraybackslash}X
%         >{\centering\arraybackslash}X
%         >{\centering\arraybackslash}X
%     }
%         \hline
%         \textbf{Scenario or Property}
%             & \ps
%             & \theterm
%             & \emph{ML} \\ \hline

%         \emph{Only obstacle-detection faulty}
%             & Guaranteed safe~\cmark
%             & Guaranteed safe~\cmark
%             & No guarantees~\xmark \\

%         \emph{Any other function faulty}
%             & No guarantees~\xmark
%             & No guarantees~\xmark
%             & No guarantees~\xmark \\

%         \emph{Operational performance} ($\uparrow$)
%             & \textcolor{deepred}{Lowest}
%             & \textcolor{deepyellow}{Medium}
%             & \textcolor{deepgreen}{Highest} \\ \hline
%     \end{tabularx}
% \end{table}

\subsubsection{Upstream Functions}
\label{sec:method:mission-to-safety:upstream}

Upstream functions lie on the same dependency branch as the Simplexed ML layer function, as illustrated in \Cref{fig:graph_model}b.  
Let $f_1 = f_{D_M A}$ denote the propagation probability from $D_M$ to the downstream node $A$, 
and let $f_2 = f_{E D_M}$ denote the propagation probability along the added upstream-to-safety-layer edge $E \leadsto D_M$.
% and let $f_2 = f_{E D_M}$ denote the propagation probability along the edge $E \leadsto D_M$ introduced from an upstream node to the safety layer node.

In the unprotected ML baseline, the safety node $D_S$ is absent, and faults may propagate directly from $D_M$ into $A$:
\begin{align}
    p_\text{ML}(A=1 \mid D_M=0, E=1) &= 0, \\
    p_\text{ML}(A=1 \mid D_M=1, E=0) &= f_1, \\
    p_\text{ML}(A=1 \mid D_M=1, E=1) &= f_1.
\end{align}
Since we condition on $D_M=1$ in the last two cases, the upstream propagation probability $f_2$ does not appear.

When \ps introduces the safety node $D_S$, the Simplex relation prevents any propagation of faults from $D_M$ or its upstream ancestors:
\begin{align}
    p_\text{PS}(A=1 \mid D_M=0, E=1) &= 0, \label{eq:ps_ut_guarantee_1}\\ 
    p_\text{PS}(A=1 \mid D_M=1, E=0) &= 0, \label{eq:ps_ut_guarantee_2}\\
    p_\text{PS}(A=1 \mid D_M=1, E=1) &= 0. \label{eq:ps_ut_guarantee_3}
\end{align}
Thus, \ps enforces that faults in either $E$ or $D_M$ cannot influence $A$.

Now consider extending \ps to \theterm by introducing an upstream use edge $(E, D_S)$.  
Let $f_3 = f_{E D_S}$ denote the fault propagation probability along this new edge.  
In this case:
\begin{align}
    p_\text{UF}(A=1 \mid D_M=1, E=0) &= 0,\\
    p_\text{UF}(A=1 \mid D_M=0, E=1) &= 0,\\
    p_\text{UF}(A=1 \mid D_M=1, E=1) &= f_1 f_3.
\end{align}

The new M2S edge enables a fault-propagation path
$E \leadsto D_S \leadsto A,$
which is \emph{structurally impossible} in \ps.  
Whenever $f_1, f_3 > 0$, a fault in $E$ can affect $A$ conditional on $D_M=1$, violating the guarantee in \Cref{eq:ps_ut_guarantee_3}.

\begin{theorem}[Upstream Use Inadmissibility]
\label{thm:upstream-inadmissible}
Under \Cref{asm:faults_graph,asm:acyclic_graph}, introducing an upstream use edge $(E, D_S)$ violates the Perception Simplex safety guarantee
\begin{align}
    p_\text{PS}(A=1 \mid D_M=1, E=1) = 0.
\end{align}
\end{theorem}

\begin{proof}
Since $E \leadsto D_M$, a fault in $E$ may corrupt $D_M$, but \ps enforces
\begin{align}
    p_\text{PS}(A=1 \mid D_M=1, E=1)=0.
\end{align}
Introducing the edge $(E, D_S)$ creates the additional propagation path 
$E \leadsto D_S \leadsto A.$
Conditioned on $D_M=1$, this yields
\begin{align}
    p_\text{UF}(A=1 \mid D_M=1, E=1) = f_1 f_3,
\end{align}
which is strictly positive when $f_1,f_3>0$ and therefore contradicts the \ps guarantee. Hence, upstream use is prohibited.
\end{proof}

Therefore, \theterm prohibits incorporating upstream ML layer outputs into the safety layer, ensuring preservation of the core \ps safety invariants.

\subsubsection{Downstream Functions}
\label{sec:method:mission-to-safety:downstream}

Downstream functions lie strictly after the Simplexed ML layer node $D_M$ in the dependency graph and therefore do not influence the correctness of $D_M$ or the safety layer node $D_S$ (\Cref{fig:graph_model}c). Introducing an M2S link from such a function to $D_S$ would create a backward edge that violates \Cref{asm:acyclic_graph} and fundamentally alters the model by enabling temporal or multi-cycle dependencies. Such connections fall outside the static, single-cycle fault-propagation model analyzed in this paper, and a rigorous treatment would require an explicit temporal extension of the dependency graph, which is beyond the scope of the present formulation.
We therefore exclude downstream-to-safety use from our admissible M2S set. A complete analysis of this case is left for future work.

% Downstream functions lie strictly after the Simplexed ML layer node $D_M$ in the dependency graph, and therefore do not influence the correctness of $D_M$ or the safety layer node $D_S$ (\Cref{fig:graph_model}c). Under the \Cref{asm:faults_graph,asm:acyclic_graph}, all information flow is feed-forward: faults originating in downstream nodes occur strictly after the point at which the Simplex decision is made.

% Introducing an M2S link from a downstream function to $D_S$ would create a backward edge that violates the acyclicity assumption (see \Cref{asm:acyclic_graph}) and fundamentally alters the model by enabling temporal or multi-cycle dependencies. Such connections fall outside the static, single-cycle fault-propagation model analyzed in this paper. A rigorous treatment would require an explicit temporal extension of the dependency graph, which is beyond the scope of the present formulation.

% We therefore exclude downstream-to-safety use from our admissible M2S set. We note, however, that if such use were interpreted as incorporating stale or cross-cycle outputs, it could introduce additional failure modes because unverifiable information, which is otherwise unable to affect the current control action, would be allowed to influence the safety layer. A complete analysis of this case is left for future work.
\section{Application to Autonomous Vehicles}
\label{sec:application_av}

\begin{figure}
    \centering
    \includegraphics[width=\columnwidth]{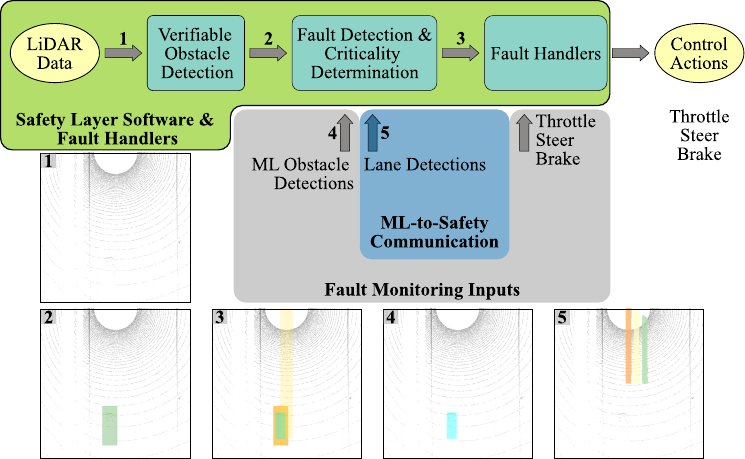}
    \caption{
    Instantiation of \theterm for AVs. 
    The safety layer performs verifiable LiDAR-based obstacle detection using the detectability model of~\cite{bansal2022verifiable} and a verified override policy~\cite{perception_simplex}. 
    Lane-detection outputs flow through the M2S link to refine obstacle criticality without weakening safety guarantees.
    }
    \label{fig:implementation}
\end{figure}

\subsection{ML Layer}

The ML layer is a high-performance autonomy stack for perception implemented using \emph{TransFuser++}~\cite{Jaeger2023ICCV}, a representative, state-of-the-art end-to-end driving model from the well-established CARLA Garage benchmark. As \theterm is agnostic to the underlying ML layer, the evaluation focuses on its coordination with the safety layer rather than model comparison. 
The end-to-end nature of \emph{TransFuser++} represents a challenging instantiation for \theterm: despite not being designed for modular interfacing, it exposes intermediate outputs that can be leveraged through the M2S link, demonstrating the practical accessibility of the \theterm framework even in non-modular settings. Evaluating \theterm with modular autonomy stacks, which natively support constraint injection and would enable full joint evaluation of both communication links, is a natural direction for future work.

% The ML layer is a high-performance autonomy stack for perception, planning, and control, implemented using the CARLA Garage \emph{TransFuser++} model~\cite{Jaeger2023ICCV}. As \theterm is agnostic to the underlying ML layer, the evaluation focuses on its coordination with the safety layer rather than model comparison.

% \emph{TransFuser++} was selected as a representative, state-of-the-art end-to-end driving model from the well-established CARLA Garage benchmark~\cite{Jaeger2023ICCV}. Its end-to-end nature represents a challenging instantiation for \theterm: despite not being designed for modular interfacing, it exposes intermediate outputs that can be leveraged through the M2S link, demonstrating the practical accessibility of the \theterm framework even in non-modular settings. Evaluating \theterm with modular autonomy stacks, which natively support constraint injection and would enable full joint evaluation of both communication links, is a natural direction for future work.

Synchronized RGB and LiDAR inputs are processed to produce object detections, lane estimates, and scene semantics, corresponding to the ML components in \Cref{sec:method}. Waypoints and control commands (steering, throttle, brake) are then generated and executed by a low-level controller. Selected outputs (\eg obstacle detections, ego velocity, control commands, and lane estimates) are exposed to the safety layer for runtime monitoring and M2S communication.

\subsection{Safety Layer}

The safety layer follows the verifiable perception framework of \ps~\cite{perception_simplex}, extended to support the interactions introduced by \theterm, and remains decoupled from the ML layer. It employs a LiDAR-based geometric detector with an analytically derived detectability model, ensuring obstacle detection whenever $y \ge kx + b$ (cf. \Cref{eq:vod_obstacle_detection_requirement}), which defines the maximum safe velocity $v_{\max}^{\mathit{safe}}$ for guaranteed stopping. The safety layer monitors ML outputs (obstacle detections, velocity, and control commands) and issues an override when a critical obstacle is missed. The safe action $a_{\text{safe}}$ applies controlled braking, ensuring collision avoidance below $v_{\max}^{\mathit{safe}}$.

% The safety layer follows the verifiable perception framework of \ps~\cite{perception_simplex}, extended to support the interactions introduced by \theterm. It is decoupled from the ML layer and provides guaranteed obstacle detection and safe overrides.

% A LiDAR-based geometric detector with an analytically derived detectability model ensures obstacle detection whenever $y \ge kx + b$ (cf. \Cref{eq:vod_obstacle_detection_requirement}). This condition defines the maximum safe velocity $v_{\max}^{\mathit{safe}}$ under which the ego vehicle can stop before reaching any detectable obstacle.

% The safety layer monitors ML outputs (obstacle detections, velocity, and control commands) and issues an override when it verifies a missed critical obstacle. The safe action $a_{\text{safe}}$ applies controlled braking, ensuring collision avoidance when operating below $v_{\max}^{\mathit{safe}}$.

\begin{assumption}[No Sensor Failure]
\label{asm:no_sensor_failure}
All sensors operate nominally, providing correct and timely measurements. 
Sensor malfunctions, such as LiDAR dropouts, camera failures, or corrupted signals, are assumed not to occur during evaluation.
\end{assumption}

This assumption is orthogonal to \theterm and outside the scope of this work. Sensor failures are assumed to be handled by state-of-the-art approaches.

\begin{assumption}[Static Obstacles]
\label{asm:static_obstacles}
All obstacles are static.
\end{assumption}

This assumption is inherited directly from \ps~\cite{perception_simplex} and is not introduced by \theterm. The contribution of \theterm is a formal framework for safe bidirectional communication between ML and safety layers. Elevating the static obstacle assumption would require extending the underlying VOD~\cite{bansal2022verifiable}, which is an orthogonal research problem outside the scope of this work.

It is worth noting that this assumption is less restrictive in practice than it may initially appear.
As observed in \ps, the safety guarantees remain applicable when obstacles are moving \emph{away} from the ego vehicle, since such motion only increases the available stopping distance.
This means that both \ps and \theterm apply directly to directed highway driving, where oncoming traffic is physically separated --- a practically significant and common real-world scenario. Handling fully dynamic environments, including adversarial or crossing agents, 
remains an open problem requiring separate treatment.
% remains an important open problem that we believe warrants dedicated treatment as a separate research contribution.
% ; we discuss this further in \Cref{sec:limitations}.

\subsection{Safety Layer to ML Layer Communication}
\label{sec:application_av:safety-to-mission}

The S2M link enables the safety layer to provide safety-verified constraints $f(C_S)$ (\eg obstacle locations) to the ML layer, allowing it to replan within the certified safety envelope instead of being immediately overridden.

In systems with constraint-aware planners, the ML layer can incorporate $f(C_S)$ to generate a safe trajectory, which is executed if it satisfies the safety constraints. Otherwise, the safety layer applies the fallback action $a_{\text{safe}}$.
This mechanism transforms the \ps override into a cooperative process, preserving formal safety while reducing unnecessary interventions.

In the CARLA-based instantiation, the \emph{TransFuser++} model does not support constraint injection due to its end-to-end nature, and S2M cannot be realized. Therefore, S2M is evaluated separately in an auxiliary experiment using a planner-based autonomy stack built on nuPlan~\cite{nuplan}.

\subsection{ML Layer to Safety Layer Communication}
\label{sec:application:mission-to-safety}

To illustrate the practical use of \theterm, we apply it to an AV system, with the safety layer and M2S link shown in \Cref{fig:implementation} and conceptually summarized in \Cref{fig:synergistic_simplex_overview}.

In our AV instantiation of \theterm, the M2S link provides the safety layer with ML outputs that are admissible under \Cref{sec:method:mission-to-safety}. 
We use \emph{lane detection}, which is independent of obstacle detection under \Cref{def:task_classes}. This independence holds in the \emph{TransFuser++}~\cite{Jaeger2023ICCV} implementation, allowing the safety layer to refine obstacle relevance without weakening \ps guarantees. 
Although the model is end-to-end and does not support constraint injection, it exposes intermediate outputs (\eg lane boundaries) that can be leveraged through M2S.

\begin{assumption}[Lane-Detection Correctness]
\label{asm:lane_correct}
The ML layer's lane-detection module provides correct lane-boundary estimates (\ie $B=0$) for the scenarios considered.
\end{assumption}

\textbf{This assumption is not required for the safety guarantees of \theterm} and does not weaken them.
As established in \Cref{sec:method:mission-to-safety:independent}, lane detection is independent of the Simplex-protected function and therefore does not affect the guarantee boundary. 
Moreover, correct lane estimation is already required for safe autonomous driving regardless of whether it is used in the safety layer, analogous to standard assumptions such as correct actuator behavior in classical Simplex designs, \eg that brakes will engage when commanded. 
In practice, lane detection is a comparatively mature perception task and is routinely relied upon by modern autonomy stacks.
We make this assumption solely to ensure that the scenarios considered in our evaluation remain within the formal guarantee envelope, avoiding the need to evaluate behavior in regions where neither \ps nor \theterm provides formal guarantees.
When lane detection is faulty, the system transitions to regions outside the formal guarantee envelope, where neither \ps nor \theterm provides guarantees. Importantly, this does not introduce new failure modes relative to \ps, but simply reflects the inherent limits of the underlying safety model.

The safety layer uses lane-boundary estimates to determine whether a verified obstacle lies in the ego lane or an adjacent lane. A three-zone policy is applied:
(1) \emph{Zone~1} (ego lane): triggers braking as in \ps;  
(2) \emph{Zone~2} (adjacent lanes): triggers lighter mitigation (\ie throttle release);  
(3) \emph{Zone~3}: ignores irrelevant obstacles.  
Compared to \ps, which treats the entire field of view as Zone~1, \theterm uses lane context to reduce unnecessary braking while preserving guarantees (\Cref{fig:zoning_policy}).

\begin{figure}
    \centering
    \includegraphics[width=\columnwidth]{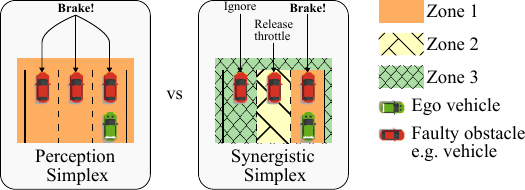}
    \caption{Zoning policy guiding how the safety layer of \theterm responds to obstacles based on lane-relative relevance.}
    \label{fig:zoning_policy}
\end{figure}

\textbf{Preservation of Safety Guarantees.}
\Cref{fig:failures_venn_diagram} visualizes the relationship between ML-layer fault regions and the system's safety guarantees under \ps and \theterm.
\ps provides formal guarantees for the green-highlighted region~1, where obstacle detection is the \emph{only} fault: in this case, the system brakes safely and avoids collision. 
As proven in \Cref{sec:method:mission-to-safety:independent}, incorporating an independent ML function into the safety layer, here lane detection, does not alter this guarantee boundary in \theterm.

If lane detection degrades, the overall system transitions to region~3. If both lane and obstacle detection degrade, it transitions to region~2. 
However, both regions lie outside the guarantee region of \ps, which covers only obstacle-detection faults. 
Thus, they are also outside the guarantees of \theterm. 

From the perspective of both \ps and \theterm, these regions are equivalent in that each involves at least one fault beyond obstacle detection, and therefore falls outside the scope of formal guarantees. 
\theterm therefore preserves the deterministic safety properties of \ps exactly, while using lane context to reduce conservativeness within the guaranteed region (region~1).

\begin{figure}
    \centering
    \includegraphics[width=\columnwidth]{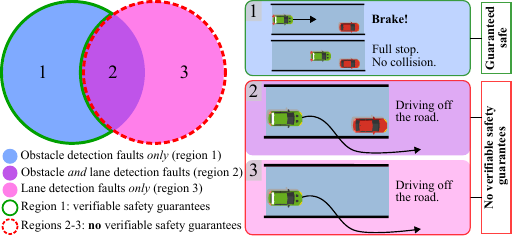}
    \caption{
    ML fault regions and their relation to the safety guarantees of \ps and \theterm.
    Region~1 (obstacle detection fault only) lies within the formal guaranteed safe region.
    Regions~2--3 lie outside the guaranteed safe region of both \ps and \theterm.
    Circle sizes are illustrative and do not reflect fault frequency.
    }
    \label{fig:failures_venn_diagram}
\end{figure}
\section{Evaluation Setup}
\label{sec:experiments}

We evaluate \theterm primarily in the CARLA simulation environment, focusing on the safety layer and the M2S link. 
Because the end-to-end policy does not support constraint injection (see \Cref{sec:application_av:safety-to-mission}), the S2M link is evaluated separately in an auxiliary experiment based on nuPlan~\cite{nuplan}.

We evaluate \theterm primarily in the CARLA simulation environment, focusing on the safety layer and the M2S link. Because the end-to-end policy does not support constraint injection (see \Cref{sec:application_av:safety-to-mission}), the S2M link is evaluated separately in an auxiliary experiment based on nuPlan~\cite{nuplan}. While this means the full bidirectional architecture is not demonstrated within a single autonomy stack, the two links address structurally independent mechanisms and their separate evaluation provides valid evidence for each. Joint evaluation using a modular autonomy stack that natively supports both intermediate output exposure and constraint injection remains an important direction for future work, consistent with the discussion in \Cref{sec:application_av}.

\subsection{Simulation Setup}
\label{sec:experiments:evaluation_environment}

\subsubsection{CARLA Evaluation Setup}
\label{sec:experiments:evaluation_environment:carla}

We use the CARLA simulator~\cite{Dosovitskiy17} with \textbf{CARLA Garage}~\cite{Jaeger2023ICCV}, which provides a modular autonomy stack and standardized scenarios. The \textit{TransFuser++} model serves as the ML layer, and we integrate the safety layer to realize the CARLA instantiation of \theterm, focusing on safety-layer behavior and M2S communication.

Experiments are conducted in controlled, traffic-free environments based on the ``Parked Obstacle'' scenario from the CARLA Leaderboard benchmark\footnote{\url{https://leaderboard.carla.org/scenarios/}}, enabling deterministic fault injection and reproducible evaluation.
All runs use consistent environmental conditions (daylight, clear weather, and good road quality). LiDAR measurements are collected at $2\,\mathrm{Hz}$. 
The evaluation setup will be open-sourced upon publication.

\subsubsection{nuPlan Evaluation Setup}
\label{sec:experiments:evaluation_environment:nuplan}

To evaluate the S2M link, we use the nuPlan simulation environment~\cite{nuplan}, which exposes a planner interface for constraint injection. In this setting, the safety layer provides verified constraints (\eg obstacle locations) to the planner, enabling replanning under S2M using a fault model consistent with the CARLA experiments.

\subsection{Models}
\label{sec:experiments:baselines_and_ss_variants}

We consider a set of autonomy architectures and variants of \theterm across the evaluation setups described above.

\paragraph{ML-Based Agent (TransFuser++) [CARLA]}

We use \textit{TransFuser++} from CARLA Garage~\cite{Jaeger2023ICCV} as the ML-only baseline, representing a high-performance learning-based policy without formal safety guarantees.

This baseline serves as a reference point for evaluating the safety benefits of \ps and \theterm. 
Because it lacks any form of fault detection or correction, its performance degradation under injected perception faults directly illustrates the vulnerability of unconstrained ML-based autonomy.
Note that the causes of ML perception faults are out of the scope of this work, and therefore fault injection is used to measure the effect of faults that ML-based systems are susceptible to~\cite{
jenn2020identifying, 
xu2021machine, 
mohseni2022taxonomy, 
araujo2024road, 
paterson2025safety, 
crash_ntsb_tesla_2016_5_7, 
crash_ntsb_uber_2018_3_18, 
% crash_ntsb_tesla_2018_3_23, 
crash_ntsb_tesla_2019_3_1, 
NHTSA_SGO_Reports
}.

\paragraph{\ps\ [CARLA]}

A formally verified baseline (\Cref{sec:related_work:perception_simplex}) that augments the ML layer with a safety layer that overrides unsafe actions based on VOD~\cite{perception_simplex,bansal2022verifiable}.

\paragraph{M2S-Only \theterm\ [CARLA]}

An ablation including only the M2S link, where the safety layer selectively incorporates ML outputs according to \Cref{sec:method:mission-to-safety}. This variant captures the core mechanism of \theterm and is evaluated in CARLA.

\paragraph{\theterm}  

The full architecture combines M2S and S2M links, enabling bidirectional interaction between the ML and safety layers. Due to architectural constraints, M2S is evaluated in CARLA and S2M in a planner-based setting.

\paragraph{S2M-Only \theterm\ [nuPlan]} 

An ablation including only the S2M link, where the safety layer provides constraint feedback for replanning. This variant is evaluated in nuPlan~\cite{nuplan}.

In summary, the primary evaluation focuses on CARLA (ML, \ps, M2S), while S2M is assessed separately in nuPlan.

\subsection{Experiments Description}
\label{sec:experiments:description}

We conduct two controlled experiments in the CARLA environment (\Cref{sec:experiments:evaluation_environment:carla}) to evaluate the safety layer and the ML-to-safety mechanism of \theterm under obstacle detection faults in the ML layer. 
Experiment~1 is additionally evaluated in nuPlan to assess the S2M link in a planner-based setting.
The configurations of both experiments are illustrated in \Cref{fig:experiments_overview}. 

\paragraph{Experiment 1: In-Lane Obstacle}

The ego vehicle approaches a static obstacle in its lane that is missed by the ML layer. This scenario evaluates collision avoidance via the safety layer. In nuPlan, it additionally tests whether S2M enables timely replanning to avoid unnecessary stopping.

\paragraph{Experiment 2: Nearby-Lane Obstacle}

An obstacle is placed in a neighboring lane closer to the ego lane and missed by the ML layer. This scenario evaluates whether the M2S link reduces unnecessary interventions by incorporating lane context, maintaining safety while improving efficiency.

\begin{figure}
    \centering
    \includegraphics[width=\columnwidth]{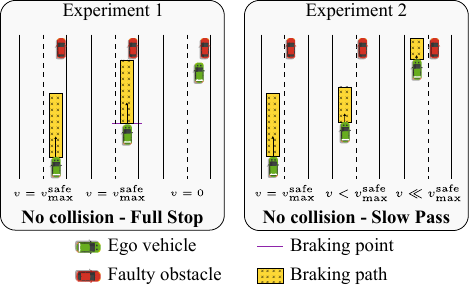}
    \caption{
    Visualization of the simulation scenarios used to evaluate \theterm. 
    Experiment~1 places an undetected obstacle in the ego lane to test the safety layer override behavior. 
    Experiment~2 places an undetected obstacle in an adjacent lane to assess whether \theterm can avoid unnecessary braking when the obstacle is close to being safety-critical.
    }
    \label{fig:experiments_overview}
\end{figure}

Under nominal behavior, routes are completed in $\sim 10\,\text{s}$. Episodes terminate upon full stop after emergency braking. Each configuration is evaluated over 10 runs.
In CARLA, we evaluate three models (ML, \ps, M2S-only \theterm) across two experiments (60 runs total). The S2M evaluation in nuPlan is reported separately in \Cref{sec:results:s2m_in_nuplan}.

\subsection{Fault Injection}
\label{sec:experiments:fault_injection}

To focus on system-level resilience rather than the specific origins of perception faults, we inject faults directly at the decision level of the ML layer.  
This procedure is applied in the CARLA-based experiments described above and targets the evaluation of the safety layer and the M2S mechanism of \theterm. 
This approach bypasses the details of fault generation (\eg sensor noise or adversarial perturbations) and instead isolates how the safety architecture responds to erroneous ML layer behavior in a controlled and reproducible manner.

Faults are induced by recording nominal ML actions in obstacle-free runs and replaying them in scenarios with obstacles, simulating missed detections that lead to unsafe behavior.

This injection strategy ensures that all evaluated models in the CARLA setting, \ie the ML-based agent, \ps, and the M2S-only variant of \theterm, operate under identical ML layer fault conditions, allowing us to attribute any differences in outcome solely to the architectures' safety mechanisms rather than the stochasticity of perception. 

\textbf{Expected Outcomes.}
The ML-based agent is expected to collide in \emph{Experiment~1} and proceed normally in \emph{Experiment~2}. 
The \ps baseline avoids collisions but may exhibit overly conservative braking. 
The M2S-enabled variant is expected to maintain zero collisions while reducing unnecessary interventions through the use of ML-layer context.

\subsection{Evaluation Metrics}
\label{sec:experiments:metrics}

To assess the safety and functional performance of each architecture, we use three quantitative metrics: \emph{Collision Rate}, \emph{Time to Completion}, and \emph{Route Completion Rate}. These metrics capture complementary aspects of performance and safety.

\paragraph{Collision Rate (CR)}  
This metric records the number of collisions that occur during a single simulation run.  
For each episode~$\tau$, we measure
$c(\tau) \in \{0,1,2,\dots\},$
the total number of impacts involving the ego vehicle.  
% Lower values indicate safer behavior and fewer violations of the safety envelope.  
% This metric directly reflects the system’s robustness to perception faults and its ability to avoid safety-critical failures.

\paragraph{Time to Completion (TC)}  
For each episode~$\tau$, we measure the time required for the ego vehicle to reach the end of the designated route without violating safety constraints.  
Let 
$t(\tau) \in \mathbb{R}_{>0}$
denote the elapsed simulation time until successful route completion.  
% If the vehicle cannot complete the route (due to an override, stall, or collision), the run is marked as incomplete.  
% Lower completion times indicate more efficient behavior and fewer unnecessary slowdowns caused by conservative safety interventions.

\paragraph{Route Completion Rate (RCR)}
For each episode~$\tau$, this metric records whether the ego vehicle reaches the goal. If the vehicle completes the route, it is marked as $\tau=1$. Otherwise it is marked as $\tau=0$.
\section{Results}
\label{sec:results}

\begin{figure}
    \centering
    \includegraphics[width=\columnwidth]{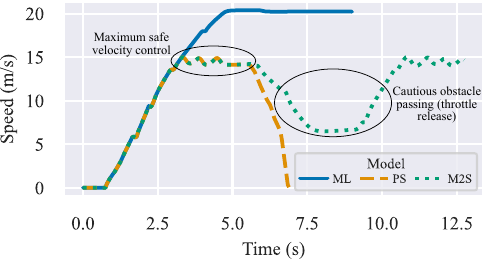}
    % \caption{
    % Example route from Experiment~2. The ML-based agent completes the route fastest by exceeding the safe velocity, but without safety guarantees. \ps exhibits conservative behavior and stalls due to emergency braking, as it overestimates a nearby-lane obstacle as safety-critical. In contrast, the M2S-enabled variant completes the route safely by applying throttle-release mitigation for obstacles outside the ego lane, thereby avoiding unnecessary emergency braking while preserving safety guarantees.
    % }
    \caption{
    % Example run from Exp.~2. The ML agent completes fastest by exceeding the safe velocity (without guarantees). \ps stalls due to conservative emergency braking, while the M2S variant completes the route safely via throttle-release mitigation.
    Example run from Exp.~2. The ML agent completes fastest by exceeding $v_\text{max}^\text{safe}$ (without guarantees). \ps stalls due to conservative emergency braking, while the M2S variant completes the route safely via throttle-release mitigation.
    }
    \label{fig:example_exp2_run}
\end{figure}

We report CARLA-based results focusing on the safety layer and the M2S mechanism of \theterm, with metrics summarized in \Cref{tab:macro_results}. 
An auxiliary S2M evaluation in a planner-based setting is presented in \Cref{sec:results:s2m_in_nuplan}. 
Representative CARLA and nuPlan scenarios are shown in \Cref{fig:screenshots_of_experiments}.

\begin{figure}
    \centering
    \includegraphics[width=\columnwidth]{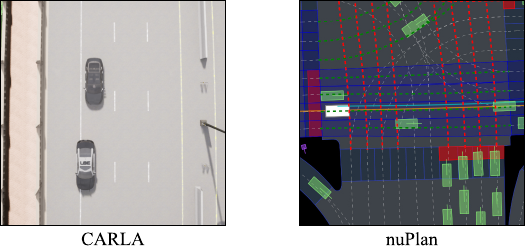}
    % \caption{
    % Representative snapshots from the CARLA (primary evaluation) and nuPlan (auxiliary S2M evaluation) scenarios.
    % }
    \caption{
    Representative snapshots from CARLA (primary evaluation) and nuPlan (auxiliary S2M evaluation), illustrating realistic closed-loop behavior in both simulation settings.
    }
    \label{fig:screenshots_of_experiments}
\end{figure}

\subsection{Experiment 1: In-Lane Obstacle}

The first experiment evaluates cases in which a faulty obstacle lies directly within the ego lane and poses an immediate collision risk. 
As expected, the ML-based agent fails in all runs: because the injected fault suppresses the in-lane obstacle, the ML layer never initiates a replan and consequently collides.

In contrast, all safety-backed architectures evaluated in CARLA, \ie \ps and the M2S \theterm, exhibit the expected fail-safe behavior. When the ML layer fails to detect a safety-critical obstacle, the safety layer triggers conservative braking, yielding $\text{CR}=0.0$ in all runs. 
However, none of these architectures completes the route (\ie $\text{RCR}=0.0$), as emergency braking halts progress once the vehicle enters the braking zone.

These results reflect a limitation of our evaluation protocol, as episodes are terminated once emergency braking brings the vehicle to a full stop, and therefore route completion is not observed even though safety is preserved. 
To assess whether S2M communication can enable continued progress under such conditions, we additionally evaluate the S2M link in a planner-based setting using nuPlan~\cite{nuplan} (see \Cref{sec:results:s2m_in_nuplan}). In that setting, where the planner can incorporate externally provided constraints and replan earlier, S2M enables successful route completion in a significant fraction of runs, demonstrating its potential when paired with a sufficiently responsive planner.

\begin{figure}
    \centering
    \includegraphics[width=\columnwidth]{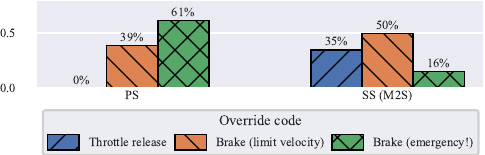}
    \caption{
    Analysis of safety layer interventions in Exp.~2. 
    The M2S-enabled variant shifts safety responses from conservative emergency braking to throttle-release mitigation, reducing unnecessary interventions while preserving safety.
    }
    \label{fig:micro_results}
\end{figure}

\subsection{Experiment 2: Nearby-Lane Obstacle}

The second experiment places faulty obstacles in a neighboring lane rather than in the ego lane, making lane context essential for correct threat assessment. As expected, none of the agents collide with the obstacle, since it does not pose a direct hazard ($\text{CR}=0.0$ for all architectures).

Under \ps, the safety layer lacks lane context and therefore treats some of these non-threatening obstacles as safety-critical due to their proximity. This conservativeness triggers unnecessary braking and premature route termination, resulting in low route-completion rates.
% \begin{figure}
%     \centering
%     \includegraphics[width=\columnwidth]{Figures/Micro Results.pdf}
%     \caption{
%     Analysis of safety layer interventions in Exp.~2. 
%     The M2S-enabled variant shifts safety responses from conservative emergency braking to throttle-release mitigation, reducing unnecessary interventions while preserving safety.
%     }
%     \label{fig:micro_results}
% \end{figure}
In contrast, the M2S-enabled variant of \theterm leverages ML-layer lane detections to classify these obstacles as less critical (Zone~2 in \Cref{fig:zoning_policy}), invoking throttle-release mitigation rather than emergency braking. Consequently, it achieves substantially higher route-completion rates than \ps, braking only in rare cases where lane detection indicates that the obstacle slightly encroaches on the ego lane.

Although the time to completion (TC) of the M2S-enabled variant is higher than that of the unconstrained ML policy, this increase reflects the deliberate use of mitigative actions under formal safety guarantees, as throttle release inevitably slows progress. Crucially, the M2S-enabled design reduces excessive conservativeness while preserving the same safety guarantees as \ps.
An illustrative route demonstrating these behavioral differences appears in \Cref{fig:example_exp2_run}.

\subsection{Analysis of Safety Layer Interventions}

To further characterize system behavior, \Cref{fig:micro_results} reports the distribution of safety layer intervention levels across all routes in Exp.~2, where the impact of the M2S mechanism is most pronounced. As expected, M2S-enabled \theterm significantly reduces the number of emergency-braking events relative to \ps, replacing them with less conservative throttle-release actions. This shift reflects the safety layer's ability to correctly identify faulty obstacles as lying outside the ego lane, enabled by the admissible use of ML-layer lane information.

\begin{table}[t]
    \centering
    \caption{
    Collision rate (CR), time to completion (TC), and route completion rate (RCR).
    TC is reported only for runs that finish the full route successfully.
    % , so no TC is shown when all agents terminate early (Exp.~1).
    % In Exp.~2, the M2S-enabled variant maintains zero collisions while completing most routes, with only a modest increase in TC relative to the ML baseline.
    }
    \label{tab:macro_results}

    \renewcommand{\arraystretch}{1.2}

    \begin{tabularx}{\columnwidth}{
        >{\raggedright\arraybackslash}X
        >{\centering\arraybackslash}X
        >{\centering\arraybackslash}X
        >{\centering\arraybackslash}X
        >{\centering\arraybackslash}X
        >{\centering\arraybackslash}X
        >{\centering\arraybackslash}X
    }
        \hline
        % Top header row
        \multirow{2}{*}{\textbf{Agent}}
        & \multicolumn{3}{c}{\textbf{Exp. 1: In-Lane Obstacle}} 
        & \multicolumn{3}{c}{\textbf{Exp. 2: Near-Lane Obstacle}} \\
        \cmidrule(lr){2-4} \cmidrule(lr){5-7}
        % Subheader row
        & \textbf{CR} & \textbf{TC, s} & \textbf{RCR}
        & \textbf{CR} & \textbf{TC, s} & \textbf{RCR} \\
        \hline

        \emph{ML}        & 1.0 & N/A    & 0.0   & 0.0 & 10.53   & 1.0 \\
        \emph{\ps}       & 0.0 & N/A    & 0.0   & 0.0 & 15.30   & 0.1 \\
        % \emph{SS(S2M)}   & 0.0 & N/A    & 0.0   & 0.0 & N/A     & 0.0 \\
        \emph{SS(M2S)}   & 0.0 & N/A    & 0.0   & 0.0 & 15.14   & 0.7 \\
        % \emph{\theterm}  & 0.0 & N/A    & 0.0   & 0.0 & 14.41   & 0.7 \\
        \hline
    \end{tabularx}
\end{table}

\subsection{S2M in NuPlan}
\label{sec:results:s2m_in_nuplan}

To evaluate the S2M link in a setting where planner-level constraint injection is supported, we conduct an auxiliary study in the nuPlan simulation environment~\cite{nuplan}. Using the same obstacle-existence fault model as in Experiment~1, the underlying nuPlan policy occasionally initiates a lane change sufficiently early, \ie before entering the safety layer's braking zone. 
In such cases, the S2M link provides safety-verified constraints in time for the planner to adjust its trajectory without triggering an emergency stop, resulting in $\text{RCR} = 74\%$ (29/39). This contrasts with the CARLA-based Experiment~1, where all runs terminate early due to late or absent replanning.

This result indicates that S2M effectiveness depends on planner responsiveness and can improve task completion while preserving safety when timely replanning is available.
\section{Conclusion}
\label{sec:conclusion}

% This paper introduced \thetermlong (\theterm), a generalization of the Simplex paradigm that enables bidirectional interaction between ML and safety layers. Unlike \ps, which relies on unilateral overrides, \theterm incorporates ML-to-safety and safety-to-ML communication to preserve deterministic safety guarantees while reducing conservativeness. Our analysis and experiments show that \theterm \textbf{retains the safety properties of \ps} while \textbf{improving performance} through coordinated use of ML outputs and safety feedback.

% These results illustrate that high performance and strong safety guarantees need not be competing objectives. 
% \theterm offers a principled foundation for scalable, verifiable runtime-assured autonomy, enabling formally unverified information to be safely leveraged for improved performance without weakening formal safety guarantees.
% This shows that formally unverified information can be incorporated into a safety architecture only under rigorously defined structural conditions, preserving the deterministic guarantees that historically motivated the strict separation between mission and safety layers.

This paper introduced \thetermlong (\theterm), a generalization of the Simplex paradigm that enables bidirectional interaction between ML and safety layers. Unlike \ps, which relies on unilateral overrides, \theterm incorporates ML-to-safety and safety-to-ML communication to preserve deterministic safety guarantees while reducing conservativeness. Our analysis and experiments show that \theterm \textbf{retains the safety properties of \ps} while \textbf{improving performance} through coordinated use of ML outputs and safety feedback. These results demonstrate that high performance and strong safety guarantees need not be competing objectives: \theterm provides a principled foundation for incorporating formally unverified information into safety-critical systems under rigorously defined structural conditions, challenging the traditional strict separation between mission and safety layers without weakening formal guarantees.

\bibliographystyle{IEEEtran}
\bibliography{references}

@inproceedings{bansal2022verifiable,
  author       = {Bansal, Ayoosh and Kim, Hunmin and Yu, Simon and Li, Bo and Hovakimyan, Naira and Caccamo, Marco and Sha, Lui},
  booktitle    = {2022 IEEE 33rd International Symposium on Software Reliability Engineering (ISSRE)},
  organization = {IEEE},
  pages        = {61--72},
  title        = {Verifiable Obstacle Detection},
  year         = {2022}
}

@article{chen2021lidar,
  author  = {Chen, Jiyang and Yu, Simon and Tabish, Rohan and Bansal, Ayoosh and Liu, Shengzhong and Abdelzaher, Tarek and Sha, Lui},
  journal = {arXiv preprint arXiv:2111.09799},
  title   = {Lidar cluster first and camera inference later: A new perspective towards autonomous driving},
  year    = {2021}
}

@inproceedings{liu2020removing,
  author       = {Liu, Shengzhong and Yao, Shuochao and Fu, Xinzhe and Tabish, Rohan and Yu, Simon and Bansal, Ayoosh and Yun, Heechul and Sha, Lui and Abdelzaher, Tarek},
  booktitle    = {2020 IEEE Real-Time Systems Symposium (RTSS)},
  organization = {IEEE},
  pages        = {319--332},
  title        = {On removing algorithmic priority inversion from mission-critical machine inference pipelines},
  year         = {2020}
}

@article{liu2021real,
  author    = {Liu, Shengzhong and Yao, Shuochao and Fu, Xinzhe and Shao, Huajie and Tabish, Rohan and Yu, Simon and Bansal, Ayoosh and Yun, Heechul and Sha, Lui and Abdelzaher, Tarek},
  journal   = {IEEE Transactions on Computers},
  number    = {8},
  pages     = {1770--1783},
  publisher = {IEEE},
  title     = {Real-time task scheduling for machine perception in intelligent cyber-physical systems},
  volume    = {71},
  year      = {2021}
}

@article{perception_simplex,
  author = {Bansal, Ayoosh and Kim, Hunmin and Yu, Simon and Li, Bo and Hovakimyan, Naira and Caccamo, Marco and Sha, Lui},
  journal = {Software Testing, Verification and Reliability},
  pages = {e1879},
  publisher = {Wiley Online Library},
  title = {Perception simplex: Verifiable collision avoidance in autonomous vehicles amidst obstacle detection faults},
  year = {2024},
  doi = {10.1002/stvr.1879},
  url = {https://doi.org/10.1002/stvr.1879},
  eprint = {https://doi.org/10.1002/stvr.1879},
  bibtex_show = true,
}

@article{simplex_original,
  author    = {Sha, Lui},
  journal   = {IEEE Software},
  number    = {4},
  pages     = {20--28},
  publisher = {Citeseer},
  title     = {Using simplicity to control complexity},
  volume    = {18},
  year      = {2001}
}

@inproceedings{simplex1,
  author    = {Tanya L. Crenshaw and
               Elsa L. Gunter and
               Craig L. Robinson and
               Lui Sha and
               P. R. Kumar},
  bibsource = {dblp computer science bibliography, https://dblp.org},
  booktitle = {Proceedings of the 28th {IEEE} Real-Time Systems Symposium {(RTSS}
               2007), 3-6 December 2007, Tucson, Arizona, {USA}},
  doi       = {10.1109/RTSS.2007.34},
  pages     = {400--412},
  title     = {The Simplex Reference Model: Limiting Fault-Propagation Due to Unreliable
               Components in Cyber-Physical System Architectures},
  url       = {https://doi.org/10.1109/RTSS.2007.34},
  year      = {2007}
}

@inproceedings{simplex2,
  address   = {San Francisco, USA},
  author    = {Stanley Bak and
               Deepti K. Chivukula and
               Olugbemiga Adekunle and
               Mu Sun and
               Marco Caccamo and
               Lui Sha},
  booktitle = {15th {IEEE} Real-Time and Embedded Technology and Applications Symposium},
  pages     = {99--107},
  title     = {The System-Level Simplex Architecture for Improved Real-Time Embedded
               System Safety},
  year      = {2009}
}

@inproceedings{hu2021exploring,
  author       = {Hu, Yigong and Liu, Shengzhong and Abdelzaher, Tarek and Wigness, Maggie and David, Philip},
  booktitle    = {2021 IEEE 27th International Conference on Embedded and Real-Time Computing Systems and Applications (RTCSA)},
  organization = {IEEE},
  pages        = {169--178},
  title        = {On exploring image resizing for optimizing criticality-based machine perception},
  year         = {2021}
}

@article{hu2022real,
  author    = {Hu, Yigong and Liu, Shengzhong and Abdelzaher, Tarek and Wigness, Maggie and David, Philip},
  journal   = {Real-Time Systems},
  pages     = {1--26},
  publisher = {Springer},
  title     = {Real-time task scheduling with image resizing for criticality-based machine perception},
  year      = {2022}
}

@inproceedings{kang2022dnn,
  author       = {Kang, Woosung and Chung, Siwoo and Kim, Jeremy Yuhyun and Lee, Youngmoon and Lee, Kilho and Lee, Jinkyu and Shin, Kang G and Chwa, Hoon Sung},
  booktitle    = {2022 IEEE 28th Real-Time and Embedded Technology and Applications Symposium (RTAS)},
  organization = {IEEE},
  pages        = {160--172},
  title        = {DNN-SAM: Split-and-Merge DNN Execution for Real-Time Object Detection},
  year         = {2022}
}

@inproceedings{liu2022self,
  author       = {Liu, Shengzhong and Fu, Xinzhe and Wigness, Maggie and David, Philip and Yao, Shuochao and Sha, Lui and Abdelzaher, Tarek},
  booktitle    = {2022 IEEE 28th Real-Time and Embedded Technology and Applications Symposium (RTAS)},
  organization = {IEEE},
  pages        = {173--186},
  title        = {Self-cueing real-time attention scheduling in criticality-aware visual machine perception},
  year         = {2022}
}

@inproceedings{tung2022irrelevant,
  author       = {Tung, Caleb and Goel, Abhinav and Hu, Xiao and Eliopoulos, Nick and Amobi, Emmanuel S and Thiruvathukal, George K and Chaudhary, Vipin and Lu, Yung-Hsiang},
  booktitle    = {2022 IEEE 4th International Conference on Artificial Intelligence Circuits and Systems (AICAS)},
  organization = {IEEE},
  pages        = {340--343},
  title        = {Irrelevant Pixels are Everywhere: Find and Exclude Them for More Efficient Computer Vision},
  year         = {2022}
}

@misc{ishii1986fault,
  author    = {Ishii, Kazuhiko and Noguchi, Atomi and Gotoh, Yoshimi},
  month     = apr # {~15},
  note      = {{US Patent 4,583,224}},
  publisher = {Google Patents},
  title     = {Fault tolerable redundancy control},
  year      = {1986}
}

@article{lala1994architectural,
  author    = {Lala, Jaynarayan H and Harper, Richard E},
  journal   = {Proceedings of the IEEE},
  number    = {1},
  pages     = {25--40},
  publisher = {IEEE},
  title     = {Architectural principles for safety-critical real-time applications},
  volume    = {82},
  year      = {1994}
}

@inproceedings{jha2019ml,
  author       = {Jha, Saurabh and Banerjee, Subho and Tsai, Timothy and Hari, Siva KS and Sullivan, Michael B and Kalbarczyk, Zbigniew T and Keckler, Stephen W and Iyer, Ravishankar K},
  booktitle    = {2019 49th Annual IEEE/IFIP International Conference on Dependable Systems and Networks (DSN)},
  organization = {IEEE},
  pages        = {112--124},
  title        = {ML-based Fault Injection for Autonomous Vehicles: A Case for Bayesian Fault Injection},
  year         = {2019}
}

@inproceedings{6629559,
  author    = {J. {Wei} and J. M. {Snider} and J. {Kim} and J. M. {Dolan} and R. {Rajkumar} and B. {Litkouhi}},
  booktitle = {2013 IEEE Intelligent Vehicles Symposium (IV)},
  doi       = {10.1109/IVS.2013.6629559},
  issn      = {1931-0587},
  month     = {June},
  number    = {},
  pages     = {763-770},
  title     = {Towards a viable autonomous driving research platform},
  volume    = {},
  year      = {2013}
}

@inproceedings{wiersma2017safety,
  author       = {Wiersma, Nils and Pareja, Ramiro},
  booktitle    = {2017 Workshop on Fault Diagnosis and Tolerance in Cryptography (FDTC)},
  organization = {IEEE},
  pages        = {9--16},
  title        = {Safety!= security: On the resilience of ASIL-D certified microcontrollers against fault injection attacks},
  year         = {2017}
}

@book{blanke2006diagnosis,
  author    = {Blanke, Mogens and Kinnaert, Michel and Lunze, Jan and Staroswiecki, Marcel and Schr{\"o}der, Jochen},
  publisher = {Springer},
  title     = {Diagnosis and fault-tolerant control},
  volume    = {2},
  year      = {2006}
}

@article{wang2018rsimplex,
	title={RSimplex: A robust control architecture for cyber and physical failures},
	author={Wang, Xiaofeng and Hovakimyan, Naira and Sha, Lui},
	journal={ACM Transactions on Cyber-Physical Systems},
	volume={2},
	number={4},
	pages={1--26},
	year={2018},
	publisher={ACM New York, NY, USA}
}

@article{mao2023sl1,
  title={SL1-Simplex: Safe Velocity Regulation of Self-Driving Vehicles in Dynamic and Unforeseen Environments},
  author={Mao, Yanbing and Gu, Yuliang and Hovakimyan, Naira and Sha, Lui and Voulgaris, Petros},
  journal={ACM Transactions on Cyber-Physical Systems},
  volume={7},
  number={1},
  pages={1--24},
  year={2023},
  publisher={ACM New York, NY}
}

@inproceedings{phan2020neural,
	title={Neural simplex architecture},
	author={Phan, Dung T and Grosu, Radu and Jansen, Nils and Paoletti, Nicola and Smolka, Scott A and Stoller, Scott D},
	booktitle={NASA Formal Methods Symposium},
	pages={97--114},
	year={2020},
	organization={Springer}
}

@inproceedings{musau2022using,
	title={On Using Real-Time Reachability for the Safety Assurance of Machine Learning Controllers},
	author={Musau, Patrick and Hamilton, Nathaniel and Lopez, Diego Manzanas and Robinette, Preston and Johnson, Taylor T},
	booktitle={2022 IEEE International Conference on Assured Autonomy (ICAA)},
	pages={1--10},
	year={2022},
	organization={IEEE}
}

@inproceedings{bogoslavskyi16iros,
  author       = {Bogoslavskyi, Igor and Stachniss, Cyrill},
  booktitle    = {2016 IEEE/RSJ International Conference on Intelligent Robots and Systems (IROS)},
  organization = {IEEE},
  pages        = {163--169},
  title        = {Fast range image-based segmentation of sparse 3D laser scans for online operation},
  year         = {2016}
}

@article{bogoslavskyi17pfg,
  author    = {Bogoslavskyi, Igor and Stachniss, Cyrill},
  journal   = {PFG--Journal of Photogrammetry, Remote Sensing and Geoinformation Science},
  number    = {1},
  pages     = {41--52},
  publisher = {Springer},
  title     = {Efficient online segmentation for sparse 3D laser scans},
  volume    = {85},
  year      = {2017}
}

@article{su2019one,
  title={One pixel attack for fooling deep neural networks},
  author={Su, Jiawei and Vargas, Danilo Vasconcellos and Sakurai, Kouichi},
  journal={IEEE Transactions on Evolutionary Computation},
  volume={23},
  number={5},
  pages={828--841},
  year={2019},
  publisher={IEEE}
}

@inproceedings{phan2017component,
  title={A component-based simplex architecture for high-assurance cyber-physical systems},
  author={Phan, Dung and Yang, Junxing and Clark, Matthew and Grosu, Radu and Schierman, John and Smolka, Scott and Stoller, Scott},
  booktitle={2017 17th International Conference on Application of Concurrency to System Design (ACSD)},
  pages={49--58},
  year={2017},
  organization={IEEE}
}

@inproceedings{mehmood2022black,
  title={The black-box simplex architecture for runtime assurance of autonomous CPS},
  author={Mehmood, Usama and Sheikhi, Sanaz and Bak, Stanley and Smolka, Scott A and Stoller, Scott D},
  booktitle={NASA formal methods symposium},
  pages={231--250},
  year={2022},
  organization={Springer}
}

@inproceedings{desai2019soter,
  title={SOTER: a runtime assurance framework for programming safe robotics systems},
  author={Desai, Ankush and Ghosh, Shromona and Seshia, Sanjit A and Shankar, Natarajan and Tiwari, Ashish},
  booktitle={2019 49th Annual IEEE/IFIP International Conference on Dependable Systems and Networks (DSN)},
  pages={138--150},
  year={2019},
  organization={IEEE}
}

@article{priority_inversion,
    title = {Priority Inheritance Protocols: An Approach to Real-Time Synchronization},
    author = {Lui Sha and Ragunathan Rajkumar and Lehoczky, \{John P.\}},
    year = "1990",
    month = sep,
    doi = "10.1109/12.57058",
    language = "English (US)",
    volume = "39",
    pages = "1175--1185",
    journal = "IEEE Transactions on Computers",
    issn = "0018-9340",
    publisher = "IEEE Computer Society",
    number = "9",
}

@inproceedings{Dosovitskiy17,
  title = { {CARLA}: {An} Open Urban Driving Simulator},
  author = {Alexey Dosovitskiy and German Ros and Felipe Codevilla and Antonio Lopez and Vladlen Koltun},
  booktitle = {Proceedings of the 1st Annual Conference on Robot Learning},
  pages = {1--16},
  year = {2017}
}

@InProceedings{Jaeger2023ICCV,
  title={Hidden Biases of End-to-End Driving Models},
  author={Bernhard Jaeger and Kashyap Chitta and Andreas Geiger},
  booktitle={Proc. of the IEEE International Conf. on Computer Vision (ICCV)},
  year={2023}
}

@ARTICLE{1335465,
  author={Avizienis, A. and Laprie, J.-C. and Randell, B. and Landwehr, C.},
  journal={IEEE Transactions on Dependable and Secure Computing}, 
  title={Basic concepts and taxonomy of dependable and secure computing}, 
  year={2004},
  volume={1},
  number={1},
  pages={11-33},
  doi={10.1109/TDSC.2004.2}
}

@misc{NHTSA_SGO_Reports,
  author = {{National Highway Traffic Safety Administration (NHTSA)}},
  title = {{Standing General Order on Crash Reporting for Automated Driving Systems and Level 2 Advanced Driver Assistance Systems}},
  howpublished = {\url{https://www.nhtsa.gov/laws-regulations/standing-general-order-crash-reporting}},
  note = {Regularly updated reports tracking incidents involving Automated Driving Systems},
  year = {2021--Present},
  urldate = {2025-11-13}
}

@INPROCEEDINGS{nuplan, 
  title={NuPlan: A closed-loop ML-based planning benchmark for autonomous vehicles},
  author={Caesar, Holger and Kabzan, Juraj and Tan, Kok Seang and Fong, Whye Kit and Wolff, Eric and Lang, Alex and Fletcher, Luke and Beijbom, Oscar and Omari, Sammy},
  booktitle={CVPR ADP3 workshop},
  year=2021
}

@inproceedings{fuller2020run,
  title={Run-time assurance: A rising technology},
  author={Fuller, Justin G},
  booktitle={2020 AIAA/IEEE 39th Digital Avionics Systems Conference (DASC)},
  pages={1--9},
  year={2020},
  organization={IEEE}
}

@article{paterson2025safety,
  title={Safety assurance of Machine Learning for autonomous systems},
  author={Paterson, Colin and Hawkins, Richard and Picardi, Chiara and Jia, Yan and Calinescu, Radu and Habli, Ibrahim},
  journal={Reliability Engineering \& System Safety},
  pages={111311},
  year={2025},
  publisher={Elsevier}
}

@article{araujo2024road,
  title={The road to safety: A review of uncertainty and applications to autonomous driving perception},
  author={Araujo, Bernardo and Teixeira, Joao F and Fonseca, Joaquim and Cerqueira, Ricardo and Beco, Sofia C},
  journal={Entropy},
  volume={26},
  number={8},
  pages={634},
  year={2024}
}

@article{xu2021machine,
  title={Machine learning for reliability engineering and safety applications: Review of current status and future opportunities},
  author={Xu, Zhaoyi and Saleh, Joseph Homer},
  journal={Reliability Engineering \& System Safety},
  volume={211},
  pages={107530},
  year={2021},
  publisher={Elsevier}
}

@article{mohseni2022taxonomy,
  title={Taxonomy of machine learning safety: A survey and primer},
  author={Mohseni, Sina and Wang, Haotao and Xiao, Chaowei and Yu, Zhiding and Wang, Zhangyang and Yadawa, Jay},
  journal={ACM Computing Surveys},
  volume={55},
  number={8},
  pages={1--38},
  year={2022},
  publisher={ACM New York, NY}
}

@inproceedings{jenn2020identifying,
  title={Identifying challenges to the certification of machine learning for safety critical systems},
  author={Jenn, Eric and Albore, Alexandre and Mamalet, Franck and Flandin, Gr{\'e}gory and Gabreau, Christophe and Delseny, Herv{\'e} and Gauffriau, Adrien and Bonnin, Hugues and Alecu, Lucian and Pirard, J{\'e}r{\'e}my and others},
  booktitle={European congress on embedded real time systems (ERTS 2020)},
  volume={1},
  year={2020}
}

@misc{crash_ntsb_tesla_2016_5_7,
  author       = {{National Transportation Safety Board}},
  howpublished = {Accident Report NTSB/HAR-17/02 PB2017-102600},
  month        = {September},
  title        = {{Collision Between a Car Operating With Automated Vehicle Control Systems and a Tractor-Semitrailer Truck Near Williston, Florida, May 7, 2016}},
  url          = {https://www.ntsb.gov/investigations/accidentreports/reports/har1702.pdf},
  year         = {2017}
}

@misc{crash_ntsb_tesla_2019_3_1,
  author       = {{National Transportation Safety Board}},
  howpublished = {Highway Accident Brief, Accident Number: HWY19FH008},
  month        = {March},
  title        = {Highway Accident Brief {HWY19FH008}},
  url          = {https://www.ntsb.gov/investigations/AccidentReports/Reports/HAB2001.pdf},
  year         = {2019}
}

@misc{crash_ntsb_uber_2018_3_18,
  author       = {{National Transportation Safety Board}},
  howpublished = {Accident Report NTSB/HAR-19/03 PB2019-101402},
  title        = {{Collision Between Vehicle Controlled by Developmental Automated Driving System and Pedestrian Tempe, Arizona, March 18, 2018}},
  url          = {https://www.ntsb.gov/investigations/AccidentReports/Reports/HAR1903.pdf},
  year         = {2019}
}

@article{pereira2020challenges,
  author    = {Pereira, Ana and Thomas, Carsten},
  journal   = {Machine Learning and Knowledge Extraction},
  number    = {4},
  pages     = {579--602},
  publisher = {MDPI},
  title     = {Challenges of machine learning applied to safety-critical cyber-physical systems},
  volume    = {2},
  year      = {2020}
}

@inproceedings{goodloe2022assuring,
  author       = {Goodloe, Alwyn E},
  booktitle    = {2022 IEEE International Symposium on Software Reliability Engineering Workshops (ISSREW)},
  organization = {IEEE},
  pages        = {326--332},
  title        = {Assuring Safety-Critical Machine Learning Enabled Systems: Challenges and Promise},
  year         = {2022}
}

@article{malleswaran2023challenges,
  author = {Malleswaran, Iswarya and Dinakaran, Shruthi},
  title  = {Challenges in Specifying Safety-Critical Systems with AI-Components},
  year   = {2023},
  url = {https://odr.chalmers.se/items/d130c7f0-55a1-454c-acbc-1caf59a8e0e1},
}

@incollection{braiek2025machine,
  author    = {Braiek, Houssem Ben and Khomh, Foutse},
  booktitle = {Trustworthy AI in Medical Imaging},
  pages     = {37--71},
  publisher = {Elsevier},
  title     = {Machine learning robustness: A primer},
  year      = {2025}
}

@article{zhang2023deep,
  title={Deep long-tailed learning: A survey},
  author={Zhang, Yifan and Kang, Bingyi and Hooi, Bryan and Yan, Shuicheng and Feng, Jiashi},
  journal={IEEE transactions on pattern analysis and machine intelligence},
  volume={45},
  number={9},
  pages={10795--10816},
  year={2023},
  publisher={IEEE}
}

@techreport{ISO26262,
  author       = {{ISO/TC 22/SC 32}},
  title        = {{Road vehicles -- Functional safety}},
  institution  = {International Organization for Standardization},
  number       = {ISO 26262:2018},
  year         = {2018},
  address      = {Geneva, Switzerland},
  type         = {Standard}
}

@techreport{DO178C,
  title        = {{DO-178C: Software Considerations in Airborne Systems and Equipment Certification}},
  institution  = {RTCA, Inc.},
  year         = {2011},
  address      = {Washington, D.C.},
  number       = {RTCA/DO-178C},
  type         = {Standard}
}

@inproceedings{NIPS2017_2650d608,
 author = {Kendall, Alex and Gal, Yarin},
 booktitle = {Advances in Neural Information Processing Systems},
 editor = {I. Guyon and U. Von Luxburg and S. Bengio and H. Wallach and R. Fergus and S. Vishwanathan and R. Garnett},
 pages = {},
 publisher = {Curran Associates, Inc.},
 title = {What Uncertainties Do We Need in Bayesian Deep Learning for Computer Vision?},
 url = {https://proceedings.neurips.cc/paper_files/paper/2017/file/2650d6089a6d640c5e85b2b88265dc2b-Paper.pdf},
 volume = {30},
 year = {2017}
}

@inproceedings{NIPS2017_9ef2ed4b,
 author = {Lakshminarayanan, Balaji and Pritzel, Alexander and Blundell, Charles},
 booktitle = {Advances in Neural Information Processing Systems},
 editor = {I. Guyon and U. Von Luxburg and S. Bengio and H. Wallach and R. Fergus and S. Vishwanathan and R. Garnett},
 pages = {},
 publisher = {Curran Associates, Inc.},
 title = {Simple and Scalable Predictive Uncertainty Estimation using Deep Ensembles},
 url = {https://proceedings.neurips.cc/paper_files/paper/2017/file/9ef2ed4b7fd2c810847ffa5fa85bce38-Paper.pdf},
 volume = {30},
 year = {2017}
}

@Article{e26080634,
AUTHOR = {Araújo, Bernardo and Teixeira, João F. and Fonseca, Joaquim and Cerqueira, Ricardo and Beco, Sofia C.},
TITLE = {The Road to Safety: A Review of Uncertainty and Applications to Autonomous Driving Perception},
JOURNAL = {Entropy},
VOLUME = {26},
YEAR = {2024},
NUMBER = {8},
ARTICLE-NUMBER = {634},
URL = {https://www.mdpi.com/1099-4300/26/8/634},
PubMedID = {39202104},
ISSN = {1099-4300},
DOI = {10.3390/e26080634}
}

@inproceedings{katz2017reluplex,
  title={Reluplex: An efficient SMT solver for verifying deep neural networks},
  author={Katz, Guy and Barrett, Clark and Dill, David L and Julian, Kyle and Kochenderfer, Mykel J},
  booktitle={International conference on computer aided verification},
  pages={97--117},
  year={2017},
  organization={Springer}
}

@inproceedings{huang2017safety,
  title={Safety verification of deep neural networks},
  author={Huang, Xiaowei and Kwiatkowska, Marta and Wang, Sen and Wu, Min},
  booktitle={International conference on computer aided verification},
  pages={3--29},
  year={2017},
  organization={Springer}
}

% \section*{Acknowledgment}

% The preferred spelling of the word ``acknowledgment'' in America is without 
% an ``e'' after the ``g''. Avoid the stilted expression ``one of us (R. B. 
% G.) thanks $\ldots$''. Instead, try ``R. B. G. thanks$\ldots$''. Put sponsor 
% acknowledgments in the unnumbered footnote on the first page.

\end{document}